%% file: main.tex
\documentclass[sigconf,nonacm]{acmart}
\setcopyright{none}
\renewcommand\footnotetextcopyrightpermission[1]{}

\usepackage{textcomp}
\usepackage{makecell}
\usepackage{multirow}
\usepackage{algorithm}
\usepackage{algpseudocode}
\usepackage{enumitem}
\usepackage{subcaption}
\usepackage{arydshln}
\usepackage{pifont}
\newcommand{\M}{SqLinear}
\newcommand{\blueval}[1]{\textcolor{blue}{\underline{#1}}}
\newcommand{\redval}[1]{\textbf{\textcolor{red}{#1}}}
\newcommand{\appref}[1]{\textbf{\underline{Appendix~\ref{#1}}}}

\begin{document}

% BLUE-EDIT (format compliance): removed manual negative spacing from the title.

\title[SqLinear]{SqLinear: Balanced Square Partitioning Makes Linear Interaction Sufficient for Large-Scale Traffic Forecasting}

% BLUE-EDIT (double-blind review): use a neutral anonymous author placeholder.
%\author{Anonymous Author(s)}
% \affiliation[obeypunctuation=true]{%
%   \institution{Zhejiang University},
%   \country{China}}
% \email{zqfang@zju.edu.cn}

%\title{SqLinear: A Linear Architecture for Large-Scale Traffic Prediction via Geometry-Adaptive Square Partitioning}

\author{Yongfeng Su}
\affiliation[obeypunctuation=true]{%
  \institution{Zhejiang University},
  \country{China}}
\email{yfsu@zju.edu.cn}

\author{Hongwen Li}
\affiliation[obeypunctuation=true]{%
  \institution{Southeast University},
  \country{China}}
\email{hwli@seu.edu.cn}

\author{Zijian Zhang}
\affiliation[obeypunctuation=true]{%
  \institution{Zhejiang University},
  \country{China}}
\email{zijian45@zju.edu.cn}

\author{Ziquan Fang\textsuperscript{\ding{41}}}
\affiliation[obeypunctuation=true]{%
  \institution{Zhejiang University},
  \country{China}}
\email{zqfang@zju.edu.cn}

\author{Lu Chen}
\affiliation[obeypunctuation=true]{%
  \institution{Zhejiang University},
  \country{China}}
\email{luchen@zju.edu.cn}

\author{Christian S. Jensen}
\affiliation[obeypunctuation=true]{%
  \institution{Aalborg University},
  \country{Denmark}}
\email{csj@cs.aau.dk}

\author{Hong Gao}
\affiliation[obeypunctuation=true]{%
  \institution{ZTE Corporation},
  \country{China}}
\email{gao.hong@zte.com.cn}

\author{Yinjun Han}
\affiliation[obeypunctuation=true]{%
  \institution{ZTE Corporation},
  \country{China}}
\email{han.yinjun@zte.com.cn}

\renewcommand{\shortauthors}{Yongfeng Su et al.}
\begin{abstract}
Traffic prediction is a core task in intelligent transportation systems (ITS) and urban-scale decision making. Despite the effectiveness of mainstream neural network-based methods, their deployment in real-world settings with thousands of traffic sensors is severely jeopardized by their poor computational scalability. To address this, the community has attempted to incorporate spatial database partitioning techniques (e.g., Grid, Quadtree, and K-D Tree) to improve model scalability. However, these approaches rely on handcrafted geometric heuristics and often produce irregular or imbalanced data partitions, leading to boundary fragmentation, excessive padding overheads, and degraded model accuracy.

In this paper, we propose \textbf{SqLinear}, an efficient and effective architecture for large-scale traffic prediction. First, we design \textbf{Square Partition}, a geometry-adaptive algorithm that partitions massive traffic sensors into balanced, non-overlapping, and compact spatial regions. Unlike existing heuristic-based designs, Square Partition is \textbf{theoretically grounded and provides provable guarantees} on partition utilization and split balance, establishing a high-quality foundation for downstream spatio-temporal modeling. Next, we propose a \textbf{Hierarchical Linear Interaction} (HLI) module that abandons the costly attention mechanisms commonly used in Transformer-based spatio-temporal models. HLI efficiently propagates global inter-region dependencies and refines them at the node level through a lightweight linear interaction scheme, enabling effective spatio-temporal modeling with linear computational complexity.
Extensive experiments on four large-scale traffic datasets and 11 baselines show that SqLinear reduces MAE by 2.30\% on average under the standard setting and by up to 6.78\% under extreme scalability settings, while reducing training runtime by 13.27\%--30.84\% in spatial- and horizon-scaling scenarios.
\end{abstract}

% \ccsdesc[500]{Information systems~Spatial-temporal systems}
% \ccsdesc[300]{Computing methodologies~Machine learning}


\maketitle

\input{sections/1_Introduction}

\input{sections/2_RelatedWork}

\input{sections/3_Preliminaries}

\input{sections/4_Method}

\input{sections/5_Experiment}

\input{sections/6_Conclusion}

%\clearpage

\balance
\bibliographystyle{ACM-Reference-Format}
\bibliography{sqlinear}

\clearpage
\appendix

\section*{Appendix Directory}
\noindent
\begin{tabular}{@{} l p{0.75\linewidth} @{}}
    \textbf{Appendix~\ref{sec:Related Work}} & \textbf{Related Work} \hfill \pageref{sec:Related Work} \\
    \addlinespace
    \textbf{Appendix~\ref{app:alg}} & \textbf{Algorithm} \hfill \pageref{app:alg} \\
    \addlinespace
    \textbf{Appendix~\ref{app:theory}} & \textbf{Theoretical Analysis} \hfill \pageref{app:theory} \\
    \hspace{4em}\ref{subsubsec:partition_analysis} & Analysis of Spatial Partition \dotfill \pageref{subsubsec:partition_analysis} \\
    \hspace{4em}\ref{subsubsection:low_rank} & Analysis of Low-rank Projection \dotfill \pageref{subsubsection:low_rank} \\
    \hspace{4em}\ref{proof:Complexity} & Complexity Analysis \dotfill \pageref{proof:Complexity} \\
    \addlinespace
    \textbf{Appendix~\ref{app:exp}} & \textbf{Additional Experiment} \hfill \pageref{app:exp} \\
    \hspace{4em}\ref{app:experimental-setup} & Detailed Experimental Setup \dotfill \pageref{app:experimental-setup} \\
    \hspace{4em}\ref{app:scalability-metrics} & Long-Horizon Prediction Performance \dotfill \pageref{app:scalability-metrics} \\
    \hspace{4em}\ref{app:Hyper-parameter} & Hyperparameter Study  (RQ5) \dotfill \pageref{app:Hyper-parameter} \\
    \hspace{4em}\ref{sec:partition_strategy} & Partition Strategy Analysis  (RQ6) \dotfill \pageref{sec:partition_strategy} \\
    \hspace{4em}\ref{app:case-study} & Case Study  (RQ7) \dotfill \pageref{app:case-study} \\
    \hspace{4em}\ref{app:limitations} & Limitations \dotfill \pageref{app:limitations} \\
\end{tabular}

\input{Appendix/4_RelatedWork}

\input{Appendix/1_Algorithm}

\input{Appendix/2_Theory}

\input{Appendix/3_Exp}

\end{document}

%% file: sections/1_Introduction.tex
\section{Introduction}

Traffic prediction plays a foundational role in urban computing applications, with the primary objective of predicting critical traffic parameters including flow, speed, and occupancy, through the analysis of historical traffic data~\cite{10.1145/3637528.3671961,10.1145/3690624.3709323,10.1145/3534678.3539396}. Traffic prediction is important for maintaining an efficient transportation system and also helps mitigate urban congestion~\cite{10.1145/3564594,10.1145/3690624.3709201,liu2024spatial}. As a result, \textbf{spatio-temporal traffic prediction} attracts considerable attention in both academia and industry~\cite{10.1609/aaai.v37i4.25555,ZengZhihao,10.1145/3637528.3671665}.

The research community has developed approaches to traffic modeling that employ different neural architectures, spanning recurrent neural networks (RNNs)~\cite{DBLP:conf/iclr/LiYS018,yang2024predicting,kong2023dynamic}, convolutional neural networks (CNNs)~\cite{wu2019graph,yu2018spa,DBLP:journals/corr/abs-2103-07719}, multi-layer perceptrons (MLPs)~\cite{shao2022spatial,yeh2024rpmixer,han2024bigst}, graph neural networks (GNNs)~\cite{wu2020connecting,lan2022dstagnn,shao2022decoupled}, and Transformer architectures~\cite{guo2019attention,zheng2020gman,fang2023spatio}. While these methods can capture complex spatio-temporal dependencies, their deployment faces scalability pressure along both the sensor dimension and the forecasting dimension: \textbf{\textit{real-world systems must process city-scale sensor networks and support forecasts beyond short-step prediction}}. However, previous studies often assume small-scale or regional traffic systems, overlooking the massive scale of real-world transportation systems. A prominent example is California's traffic monitoring system, which incorporates nearly 20,000 sensor nodes~\footnote{{\url{https://pems.dot.ca.gov}}}. At this scale, \textbf{\textit{mainstream Spatio-Temporal Graph Neural Networks (STGNNs) and Transformer-based models become computationally infeasible}}: modeling global interactions among 20,000 nodes at quadratic complexity results in hundreds of millions of pairwise interactions per time step, and the cost is further multiplied across long historical windows and forecasting horizons, leading to prohibitive GPU memory consumption. As a consequence, it is difficult to deploy these models in city-scale transportation systems.

To address scalability challenges in large-scale prediction settings, \textbf{spatial database partitioning} is a widely adopted strategy~\cite{fang2024efficient, yuan2018hetero}. Existing techniques generally fall into two categories: \textbf{\textit{Road Network-based Methods}} that cluster traffic sensors based on topological connectivity, and \textbf{\textit{Free Space-based Methods}} that partition traffic sensors according to geographic coordinates. Road network-based approaches explicitly preserve graph structure, but they often incur substantial computational overhead and generate highly irregular subgraphs with skewed sizes and shapes. Such irregularity complicates parallel processing and capacity-bounded modeling, limiting their scalability in practice~\cite{han2024topology, chiang2019cluster}. Consequently, recent studies~\cite{an2024lisa,10.1145/3637528.3671662} have shifted toward Free Space-based strategies (e.g., Grid~\cite{10.1145/356789.356797}, Quadtree~\cite{samet1984quadtree}, K-D Tree~\cite{10.1145/361002.361007}), which originate from the data management literature and offer lightweight means of grouping large numbers of traffic sensors into compact spatial regions suitable for scalable spatio-temporal modeling. In this study, we also focus on the free space setting, as it represents a fundamental and widely adopted abstraction for large-scale spatial data management~\cite{li2020just,li2023trajmesa}. Importantly, our approach is not restricted to free space and can be applied to \textbf{road network settings} (\textbf{Sec.~\ref{ScalabilityAnalysis}}).

Despite previous efforts on leveraging spatial partitioning to enhance prediction efficiency, we observe critical limitations. As shown in Figs.~\ref{intro:model}(a) and (b), \textbf{Grid-based} and \textbf{Quadtree-based} methods yield partitions with both imbalanced node distributions and excessive empty regions. As shown in Fig.~\ref{intro:model}(c), despite their adaptive splitting strategy, \textbf{K-D Tree-based} methods generate elongated patches that fail to preserve the spatial neighborhood structures of traffic nodes~\cite{samet2006foundations}. These deficiencies not only reduce the potential computational benefits but also reduce the abilities of models to capture complex spatial dependencies. \textbf{\textit{Overall, achieving effective spatial partitioning has emerged as a fundamental requirement for realizing efficient large-scale traffic prediction.}} This requirement leads us to identify two key challenges below.

\begin{figure}[t]
  \centering
   %\vspace{-0.1cm}
    \includegraphics[width=\columnwidth]{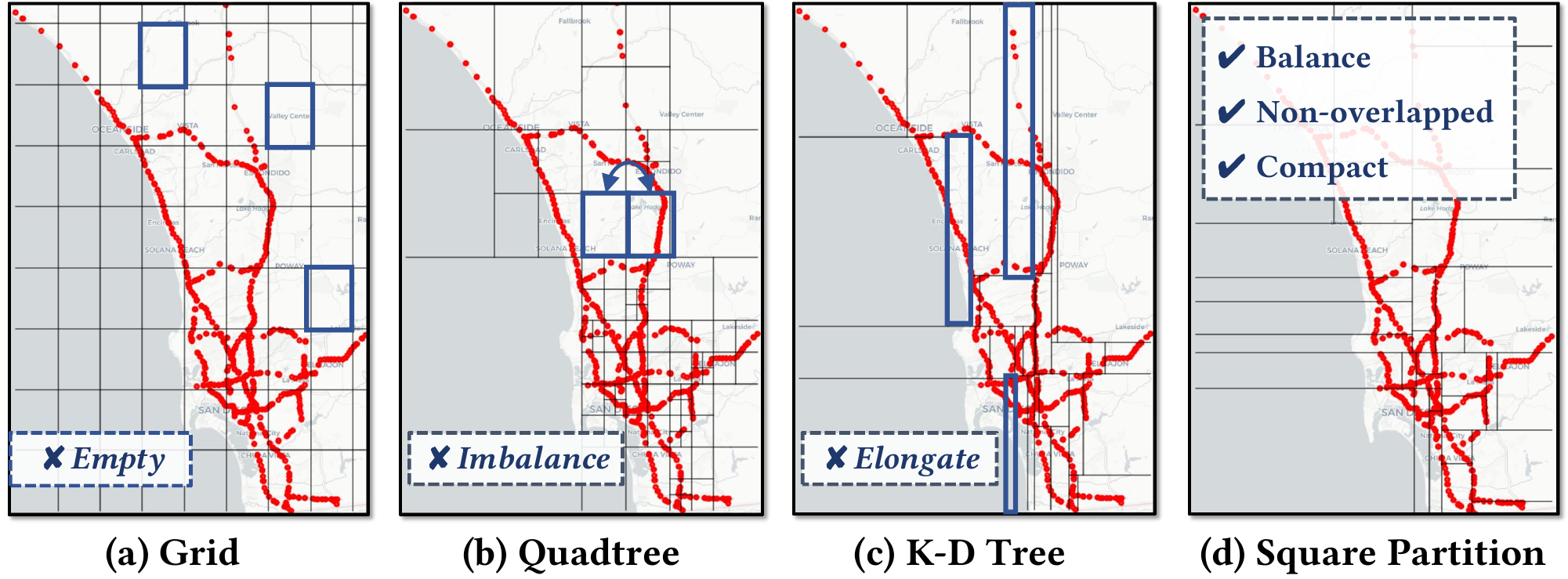}
    \includegraphics[width=0.98\columnwidth]{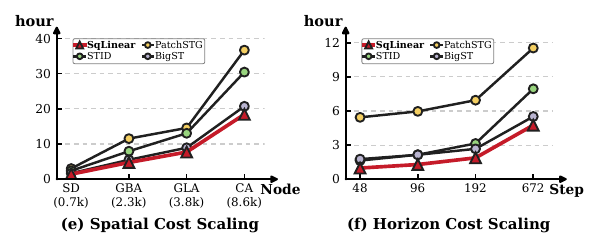}
    \vspace{-3mm}
  \caption{Top: spatial partitioning quality; Bottom: training cost scaling with node count and prediction horizon.}
   \vspace{-5mm}
  \label{intro:model}
\end{figure}

\textbf{\textit{Challenge I: How to generate high-quality spatial partitions to improve spatio-temporal modeling efficiency?}} While existing partitioning approaches show promise for large-scale data processing, employing these methods directly for traffic prediction is suboptimal due to their rigid geometric partitioning schemes. As shown in Fig.~\ref{intro:model}, \textit{Whether based on Grid, Quadtree, or K-D Tree partitioning, existing methods employ splitting strategies that are unable to contend well with the underlying traffic data distribution characteristics}. Specifically, (i) Grid-based methods enforce uniform partitions regardless of node density variations, (ii) Quadtree-based methods employ fixed midpoint quadrant divisions, and (iii) K-D Tree-based methods implement alternating median splits along the coordinate axes. %These inflexible approaches inevitably generate partitions with either severe node imbalance or geometrically irregular shapes. 
While some solutions~\cite{fang2024efficient} utilize padding operations, this introduces undesirable computational overheads without addressing the core shape irregularity problem. For example, PatchSTG ~\cite{fang2024efficient} implementation on the CA dataset~\cite{liu2024largest} requires padding 8,600 nodes across 512 patches to a uniform size of 24, resulting in a 43\% computational overhead. More importantly, these suboptimal partitions can weaken local spatial context, creating a bottleneck for prediction performance. According to \textit{the First Law of Geography}~\cite{tobler1970computer}, traffic patterns are most strongly correlated within compact local neighborhoods. This insight motivates compact spatial patches that preserve short-range locality while still supporting efficient tensorization. \textit{The key challenge is to achieve a geometry-adaptive partitioning scheme that simultaneously achieves i) balanced node distribution across patches, ii) elimination of the need for padding, and iii) compact spatial layouts that reduce severe elongation.}

\textbf{\textit{Challenge II: How to enable scalable spatio-temporal traffic prediction?}} Although spatial partitioning reduces the number of modeling units from $N$ to $P$ ($N$ and $P$ denote the number of traffic nodes and partitions, respectively), there are three modeling challenges. (i) \textbf{Topology recovery}. Since partitions are primarily based on geometric coordinates rather than road networks, continuous arterial corridors may be split across partitions, or topologically unrelated road segments may appear in the same partition. As a result, the model must propagate information globally across patches without relying on expensive graph computations. (ii) \textbf{Patch-scale computation}. Even after partitioning, large-scale networks can generate hundreds to thousands of patches $P$, posing significant computational challenges for Transformer-based models. For instance, PatchSTG~\cite{fang2024efficient} attempts to implicitly reconstruct severed topological connections between patches via global self-attention, which exhibits quadratic complexity $O(P^2)$. With the CA dataset partitioned into $P=512$ patches, each attention layer must compute and store $P \times P = 262,144$ pairwise attention scores. As illustrated in Fig.~\ref{intro:model}(e), with increasing numbers of traffic nodes, existing approaches incur high computational costs due to their quadratic complexity. 
(iii) \textbf{Temporal scalability}. Beyond spatial scale, practical forecasting often extends from short-step prediction to longer forecasting ranges. However, as reported in Fig.~\ref{intro:model}(f), existing methods suffer a significant drop in efficiency as the prediction horizon expands. A scalable model should therefore preserve robust spatio-temporal representations without relying on increasingly costly attention or graph propagation modules. \textit{Overall, the key challenge is to develop a scalable spatio-temporal modeling framework that i) effectively captures both per-node local dynamics and global cross-patch relations, ii) operates with sub-quadratic complexity to scale efficiently to thousands of spatial patches, and iii) maintains favorable accuracy-efficiency trade-offs under extended forecasting ranges.}

In this paper, we propose SqLinear, an effective and efficient \underline{Linear} architecture for large-scale traffic forecasting via geometry-adaptive \underline{Sq}uare partitioning. 
To overcome \textbf{\textit{Challenge I}}, we design \textbf{Square Partition} as a traffic-oriented spatial patching scheme rather than directly employing a general-purpose spatial index. It selects splitting axes dynamically according to maximum span dimensions and chooses capacity-aware split positions to balance node counts across patches. As shown in Fig.~\ref{intro:model}(d), this design encourages compact, non-overlapping spatial patches without requiring padding, thereby providing scalable tensor units for downstream traffic forecasting. 
To address \textbf{\textit{Challenge II}}, we develop a \textbf{Hierarchical Linear Interaction (HLI)} module, a computationally efficient neural architecture alternative to attention mechanisms. By decomposing spatial dependencies into global inter-patch communication and per-node local refinement, HLI reduces quadratic complexity to linear complexity while maintaining modeling capacity.
%\textcolor{blue}{For spatial scalability, Fig.~\ref{intro:model}(e) compares training costs across networks of increasing size under the longest forecasting horizon of 672 steps. For temporal scalability, Fig.~\ref{intro:model}(f) reports the training cost on the largest CA dataset as the forecasting horizon increases.}
In summary, this paper makes the following contributions:
\begin{itemize}[leftmargin=*]
    \vspace{-0.1mm}
    \item \textbf{High-Quality Partitioning.} We propose a traffic-oriented spatial patching scheme that organizes irregular sensor nodes into balanced, non-overlapping, and compact patches without padding operations required by existing methods, establishing a scalable data layout for large-scale traffic prediction.

    \item \textbf{Hierarchical Linear Modeling.} We develop a hierarchical linear interaction module that efficiently captures inter-patch spatio-temporal dependencies and refines them per node, reducing the quadratic complexity of attention-based approaches to linear complexity while maintaining modeling fidelity.
 
    \item \textbf{Theoretical Analysis.} We analyze the partition utilization, node balance, compactness, and computational complexity of SqLinear, clarifying why the proposed data layout and linear interaction design support scalable forecasting.

    \item \textbf{Extensive Experiments.} We evaluate SqLinear against 11 competitive baselines across four large-scale datasets. The results demonstrate SqLinear's state-of-the-art accuracy, favorable accuracy-efficiency trade-offs, and superiority when scaling along both sensor-network size and prediction length.
\end{itemize}

%The rest of the paper is organized as follows. We present preliminaries in Section~\ref{sec:preliminaries}. Section~\ref{sec:method} describes the method. Experimental findings are reported in Section~\ref{sec:Experiments}. Finally, we conclude the paper in Section~\ref{sec:CONCLUSION}. Due to space limitations, the \textbf{Related Work} section is provided in \appref{sec:Related Work}. 

%% file: sections/2_RelatedWork.tex
% \section{Related Work}
% We provide a detailed review of related work on spatio-temporal prediction and patching strategies in deep learning in \underline{\textbf{Appendix~\ref{sec:Related Work}}}.

\section{Related Work}
A comprehensive review of related work on spatio-temporal prediction and patching strategies in deep learning is deferred to \underline{\textbf{Appendix~\ref{sec:Related Work}}}. To facilitate quick comparison, Table~\ref{Related:comp} highlights the main distinctions between our approach and the most closely related methods.

\begin{table}[tb]
    \aboverulesep=0ex
    \belowrulesep=0ex
    \caption{Comparison of SqLinear with related studies.}
    \vspace{-4mm}
    \centering    
    \resizebox{1.0\linewidth}{!}{
    \begin{tabular}{c|ccc|c} 
    \toprule
        Model & \multicolumn{1}{c}{\makecell[c]{Large-Scale\\Input}} & \multicolumn{1}{c}{\makecell[c]{Spatial\\Management}}& \multicolumn{1}{c}{\makecell[c]{Scalable Data-Layout\\Analysis}}   & \multicolumn{1}{|c}{\makecell[c]{Complexity\\w.r.t. $N$}}\\
        \midrule
        AGCRN [\textit{NIPS} 20] & \ding{56} & \ding{56} & \ding{56} & $O(N^2)$\\
        D2STGNN [\textit{VLDB} 22]  & \ding{56} & \ding{56} & \ding{56} & $O(N^2)$\\
        BigST [\textit{VLDB} 24] & \ding{52} & \ding{56} &  \ding{56} & $O(N)$\\
        PatchSTG [\textit{KDD} 25] & \ding{52} & \ding{52} &\ding{56} & $O(N\sqrt{N})$\\
        \midrule
        \makecell[c]{\textbf{SqLinear}} & \ding{52} & \ding{52} & \ding{52} & $O(N)$\\
        \bottomrule
    \vspace{-5mm} 
    \end{tabular}}
    \label{Related:comp}
\end{table}

%% file: sections/3_Preliminaries.tex
\section{Preliminaries}
\label{sec:preliminaries}

A traffic monitoring system consists of $N$ fixed sensor nodes distributed over a road network, where each node corresponds to a physical observation point, such as a loop detector, camera station, or road segment sensor.

\textit{\textbf{Definition (Spatio-Temporal Data)}}. At each discrete time step, every sensor records $D$ traffic variables, such as traffic flow, average speed, and occupancy. We represent the resulting spatio-temporal observations as a tensor $X \in \mathbb{R}^{T \times N \times D}$, where $T$ is the number of time steps and $X_{t,n,:}$ denotes the traffic state observed by sensor node $n$ at time step $t$. Each sensor node is associated with geospatial coordinates, specified by its latitude $Lat$ and longitude $Lng$.

\textit{\textbf{Problem Formulation (Traffic Prediction)}} Given a historical traffic sequence of the past $H$ time steps, denoted as $X_H = [X_{t-H+1}, X_{t-H+2}, \dots, X_{t}] \in \mathbb{R}^{H \times N\times D}$, the objective is to learn a mapping function $f(\cdot)$ parameterized by learnable parameters $\theta$. The function forecasts the traffic sequence for the next $F$ time steps, denoted as $Y_F = [Y_{t+1}, Y_{t+2}, \dots, Y_{t+F}] \in \mathbb{R}^{F \times N \times D}$. This process can be expressed as follows.

\begin{equation}
[X_{t-H+1}, X_{t-H+2}, \dots, X_{t}] \xrightarrow[\theta]{f(\cdot)} [Y_{t+1}, Y_{t+2}, \dots, Y_{t+F}],
\end{equation}
where each $X_i \in \mathbb{R}^{N\times D}$ or $Y_i \in \mathbb{R}^{N\times D}$ represents the traffic states of all $N$ sensor nodes at one time step. 

%The prediction target, therefore, preserves the sensor-node granularity of the original traffic monitoring system.

Note that, given historical observations from all sensor nodes, the model predicts the future traffic state of each sensor node rather than an aggregate value for a road network or geographic area.

%% file: sections/4_Method.tex
\begin{figure*}[t]
  \centering
   \vspace{-3mm}
\includegraphics[width=\linewidth]{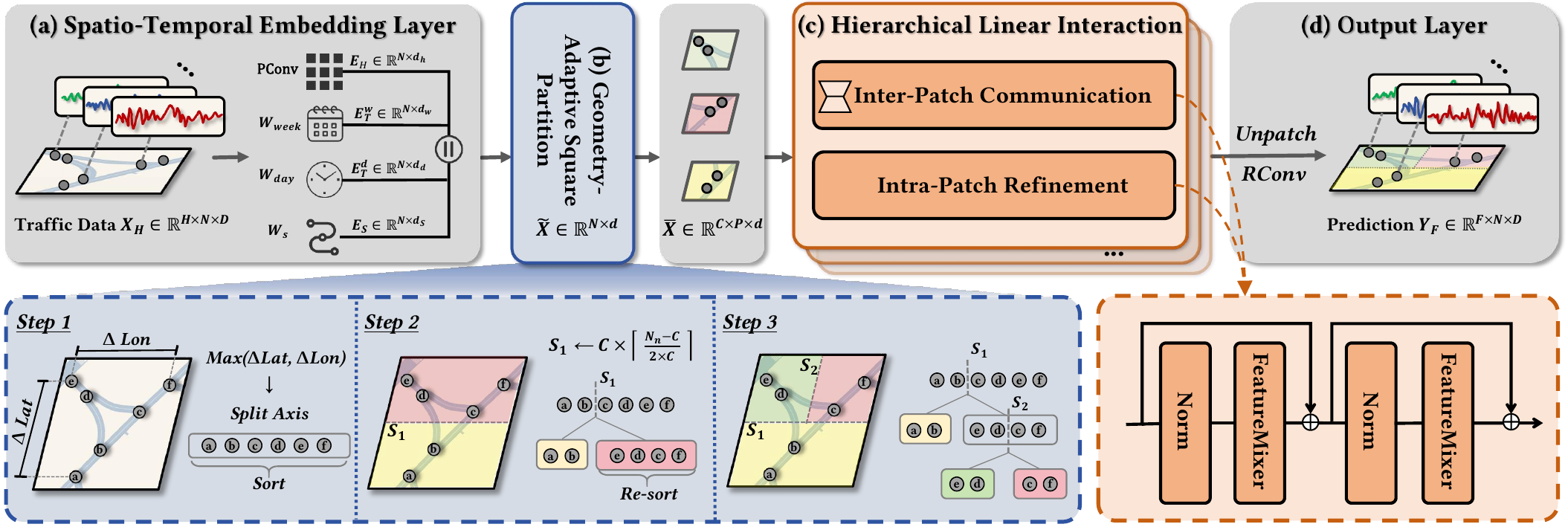}
  \caption{The architecture and workflow of \M.}
\label{method:modelfra}
\end{figure*}

\section{Methodology}
\label{sec:method}

The SqLinear architecture, depicted in Fig.~\ref{method:modelfra}, encompasses four main modules: (a) Spatio-Temporal Embedding Layer, (b) Geometry-Adaptive Square Partition, (c) Hierarchical Linear Interaction, and (d) Output Layer. We proceed to first elaborate on each module and then present a theoretical analysis of the method. %See \appref{app:alg} for the complete workflow pseudocode of SqLinear.

\subsection{Spatio-Temporal Embedding Layer}
The Spatio-Temporal Embedding Layer is illustrated in Fig.~\ref{method:modelfra}(a). First, we embed tokens through a temporal convolution. Specifically, the input data $X_H$ is transformed into the embedding $E_H \in \mathbb{R}^{N \times d_h}$:
\begin{equation}
  E_H = \mathit{PConv}([X_H \,\vert\vert\, \tau_d/T_d \,\vert\vert\, \tau_w/T_w]; \Theta_h),
  \label{eq: token emb}
\end{equation}
where $\mathit{PConv}$ denotes a convolution with kernel size and stride $s_t$ along the temporal axis ($s_t=H$ in our setting); $\tau_d,\tau_w$ are the time-of-day and day-of-week indices of the input steps; $N$ is the number of nodes, $d_h$ is the hidden dimensionality of $E_H$, and $\Theta_h$ denotes the learnable parameters of PConv.

To preserve the temporal information in the tokens, we maintain two learnable embedding tables and index them using the time-of-day and day-of-week index of the most recent input step, $\tau_d,\tau_w$. The time-of-day embedding $E_T^{d} \in \mathbb{R}^{N \times d_d }$ and day-of-week embedding $E_T^{w}\in \mathbb{R}^{N \times d_w }$ are retrieved as follows:
\begin{equation}
  E_T^{d} = W_{day}[\tau_d],\quad
  E_T^{w} = W_{week}[\tau_w],
\label{eq: te_2}
\end{equation}
where $W_{day} \in \mathbb{R}^{T_d \times d_d}$ and $W_{week} \in \mathbb{R}^{T_w \times d_w}$ are the learnable parameter embeddings for the time-of-day and day-of-week, respectively. Here, $T_d$ is the number of time-of-day slots, $T_w$ is the number of days in a week, and $d_d$ and $d_w$ are the numbers of dimensions of the respective temporal embeddings. To represent node-specific spatial characteristics, we maintain a learnable node identity embedding~\cite{shao2022spatial}: $E_S = W_s$, where $W_s \in \mathbb{R}^{N \times d_s}$ is a freely learnable parameter matrix and $d_s$ is the hidden dimensionality for the spatial embedding. Finally, we concatenate the above embeddings to obtain the spatio-temporal embedding:

\begin{equation}
    \tilde{X} =  E_H \vert\vert E_T^{d}\vert\vert E_T^{w}\vert\vert E_S,
\end{equation}
where $\tilde{X}\in\mathbb{R}^{N\times d}$, $d=d_h+d_d+d_w+d_s$, and $\vert\vert$ denotes concatenation along the feature dimension.

\subsection{Geometry-Adaptive Square Partition}
We propose the Geometry-Adaptive Square Partition to organize spatial nodes into distinct sub-patches recursively. ``\textbf{Geometry-adaptive}'' means the splitting axis and splitting position are determined by the distribution of sensor nodes, rather than by fixed grid boundaries, Quadtree midpoints, or alternating K-D Tree axes.

The process begins with an entire region containing $N_n$ nodes and continues to partition it until the number of nodes in any resulting sub-patch is at most a predefined leaf node capacity $C$. For instance, the example in Fig.~\ref{method:modelfra}(b) shows an initial region of $N_n$=6 nodes being processed with capacity $C$ = 2. This method is designed to encourage compact spatial patches while keeping node counts balanced, thereby reducing severe elongation and padding overhead in downstream tensor processing. The algorithm proceeds in three main steps, as detailed below.
% \begin{figure}[t]
%   \centering
%     \includegraphics[width=0.8\linewidth]{pics/SP.pdf}
%   \caption{The Square Partition Algorithm.}
%   \label{method:datt}
% \end{figure}

\textbf{\textit{\underline{Step 1}: Determination of the Splitting Axis.} }
Given a set of nodes $N$ within a defined region, where each node $n \in N$ has coordinates $(n_{\mathrm{Lng}}, n_{\mathrm{Lat}})$, the spatial extent in longitude and latitude is first computed. The longitudinal span $\Delta_{\mathrm{Lng}}$, and the latitudinal span $\Delta_{\mathrm{Lat}}$, are defined as $
\Delta_{\mathrm{Lat}} = \max_{n \in N}(\mathrm{Lat}_{n}) - \min_{n \in N}(\mathrm{Lat}_{n}), \Delta_{\mathrm{Lng}} = \max_{n \in N}(\mathrm{Lng}_{n}) - \min_{n \in N}(\mathrm{Lng}_{n})
$, where $\mathrm{Lat}_{n}$ and $\mathrm{Lng}_{n}$ are the latitude and longitude of a node $n$, respectively. The axis with the larger span is selected as the splitting axis for the current partition:

\begin{equation}
\mathit{Split Axis} =
\begin{cases}
\mathrm{Longitude}\quad\text{if } \Delta_{\mathrm{Lng}} \geq \Delta_{\mathrm{Lat}} \\
\mathrm{Latitude}\quad \text{else}
\label{splitaxis}
\end{cases}
\end{equation}

It splits along the region's longest dimension to reduce severe elongation and encourage compact patches.

\textbf{\textit{\underline{Step 2}: Identification of the Splitting Hyperplane.}} Once the splitting axis is determined, a capacity-aware splitting position is established along this axis to partition the nodes into two subsets with controlled size imbalance. The nodes are organized via a single-key sorting operation, arranged in ascending order of their scalar coordinate values along the splitting axis. The split point $S_p$, which corresponds to an index in the sorted list of nodes, is calculated as follows:
\begin{equation}
S_p = C\left\lceil\frac{N_n-C}{2C}\right\rceil,
\label{splitpoint}
\end{equation}
where $\lceil \cdot \rceil$ is the ceiling function. This formulation sets one child patch to contain $S_p$ nodes and the other to contain $N_n-S_p$ nodes. The resulting split imbalance is bounded by a constant controlled by the leaf capacity $C$ (specifically, at most $C$ as analyzed in \textbf{Theorem~\ref{theorem:Node Balance Degree}}), thus achieving a high degree of balance for recursive tensor construction. If $N_n \le C$, the region is not partitioned further and constitutes a leaf node.

\textbf{\textit{\underline{Step 3}: Recursive Subdivision.}} The node set of the current patch is divided into two patches based on the position of the splitting hyperplane. This process is then applied recursively to each of the newly formed patches, independently determining their respective splitting axes and hyperplanes. The recursion continues until the number of nodes in each patch is at most the specified capacity $C$, at which point the spatial partitioning is complete. The overall process is formulated as follows:
\begin{equation}
\mathrm{Index} = \text{BFS}(\mathrm{Square}(\mathrm{Lat},\mathrm{Lng},C)),
\label{BFS}
\end{equation}
where $\mathrm{Index}\in\mathbb{R}^{N}$ is the resulting index, $\mathrm{Square}(\cdot)$ and BFS$(\cdot)$ denote the Square Partition construction and breadth-first search operation, and $\mathrm{Lat}$ and $\mathrm{Lng}$ represent the sets of latitude and longitude coordinates of all nodes.

Using the constructed balanced and non-overlapping spatial patch indices, the spatio-temporal embedding $\tilde{X}$ undergoes a patching operation defined as follows:
\begin{equation}
\bar{X} = \text{Patching}(\mathit{Index}, \tilde{X})
\end{equation}

The patching operation only reorders and groups the node dimension according to $\mathit{Index}$; it does not split the concatenated historical, temporal, and spatial feature dimensions. Since $C$ divides $N$ in our experimental settings, the patching is padding-free, yielding $\bar{X} \in \mathbb{R}^{C \times P \times d}$, where $C$ denotes the number of nodes within each patch and $P$ denotes the number of patches. See \appref{app:alg} for the complete workflow pseudocode of Square Partition.

% BLUE-EDIT: synchronized the main-text claims with the corrected theoretical results.

\textbf{Theoretical Analysis.} We next present the main theoretical guarantees of Square Partition. For brevity, we provide only the key results here, while complete proofs of all theorems are deferred to \appref{app:theory}. Square Partition is conceptually related to balanced K-D Trees~\cite{10.1145/361002.361007} and balanced-aspect-ratio (BAR) trees, but differs fundamentally in both splitting and optimization objective.

i) Unlike classic K-D Trees, which alternate the splitting axis according to tree depth regardless of the region geometry and may repeatedly partition the shorter side, Square Partition always splits along the \emph{longest-span} axis. Consequently, it never subdivides the shorter dimension and therefore avoids systematically amplifying geometric elongation. This property is formalized in \textbf{Proposition~\ref{prop:axis-selection}}.

ii) Compared with BAR-trees, which control aspect ratios through recursive ``donut'' cuts at the expense of node-count balance, Square Partition integrates longest-axis splitting with capacity-aware split-point selection (Eq.~\ref{splitpoint}). This design simultaneously guarantees bounded partition imbalance (see \textbf{Theorem~\ref{theorem:Node Balance Degree}}), padding-free tensorization (see \textbf{Theorem~\ref{theorem:Patch Utilization}}), and avoidance of forced shorter-axis cuts (\textbf{Proposition~\ref{prop:axis-selection}}).

iii) More importantly, unlike general-purpose spatial indexing structures designed for range and $k$NN queries, Square Partition is specifically tailored for tensorized trajectory forecasting. By enforcing a capacity constraint on every leaf node, each partition can be directly mapped to a padding-free tensor in $\mathbb{R}^{C\times P\times d}$, thereby aligning the partitioning objective with the computational requirements of downstream forecasting models.

\subsection{Hierarchical Linear Interaction}
To efficiently model the complex spatial dependencies in traffic flow data, we propose a novel Hierarchical Linear Interaction (HLI) Block. As illustrated in Fig.~\ref{method:modelfra}(c), this block separates spatial modeling into two stages. First, it performs \textbf{Inter-Patch Communication} between different regions to capture global, long-range dependencies. Second, it performs \textbf{Intra-Patch Refinement}, fusing this global context into each node's own representation. This macro-to-micro design allows each node to integrate network-wide context at low computational cost.

The Hierarchical Linear Interaction blocks are stacked in $L$ structurally identical layers. We therefore describe the computational process at the $l$-th layer ($1 \leq l \leq L$).

\subsubsection{Inter-Patch Communication.} 
This module captures global dependencies between patches. It first groups nodes that share the same within-patch index across all patches and processes initial feature interactions via a feature mixer:
\begin{equation}
H_{\text{inter}}^{(l)} = \text{FeatureMixer}^{(l)}\left(\text{Norm}^{(l)}(\bar{X}^{(l-1)})\right),
\end{equation}
where $\text{Norm}^{(l)}$ is layer normalization, $\text{FeatureMixer}^{(l)}$ is a lightweight GELU-activated MLP applied position-wise along the feature dimension, $H_{\text{inter}}^{(l)}$ denotes the output from $\text{FeatureMixer}^{(l)}$, and $\bar{X}^{(l-1)}$ denotes the output from layer $l-1$. 
% A learnable inter-patch projection matrix then explicitly models influence weights between patches:
% \begin{equation}
% \bar{H}_{\text{inter}}^{(l)} = f\left(H_{\text{inter}}^{(l)}; \Theta_{\text{inter}}^{(l)}\right)
% \end{equation}
% where $\bar{H}_{\text{inter}}^{(l)}$ denotes the output, $f(\cdot)$ denotes the linear transformation parameterized by $\Theta_{\text{inter}}^{(l)} \in \mathbb{R}^{P \times P}$, which encodes directional influence weights between patches.
\begin{figure}[t]
    \centering
    \includegraphics[width=0.9\columnwidth]{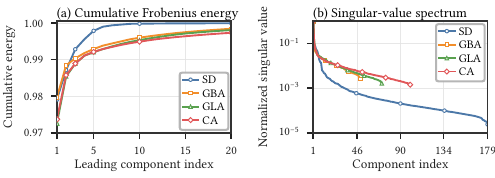}
      \vspace{-5mm}
    \caption{Low-rank measurement.}
    \label{Pic:Low-rank}
    \vspace{-5mm}
\end{figure}

Large-scale traffic systems often exhibit spatial redundancy, as nearby areas may share synchronized temporal dynamics, suggesting that patch-wise dependencies can be approximated by low-rank structures. To examine this using training data only, we compute signed Pearson dependency matrices by treating patches as variables and stacking time steps and within-patch indices as observations, matching the shared inter-patch operator used by HLI. As shown in Fig.~\ref{Pic:Low-rank}(a), the cumulative Frobenius energy rapidly approaches one: the leading component captures 97.3\%--98.0\% across the four datasets, and both the 90\% and 95\% thresholds are reached at rank 1. Fig.~\ref{Pic:Low-rank}(b) shows the corresponding long-tail decay of the normalized singular values. These observations provide empirical motivation for compactly modeling patch-wise dependencies. Accordingly, to capture the global relationships among $P$ patches, we model their influence weights using a low-rank projection:

\begin{equation}
\bar{H}_{\text{inter}}^{(l)} = f\left(H_{\text{inter}}^{(l)}; \Theta_{\text{inter}}^{(l)}\right),
\end{equation}
where $f(\cdot)$ denotes the linear transformation parameterized by $\Theta_{\text{inter}}^{(l)}$ that consists of two low-rank matrices $\Theta_{r1}^{(l)} \in \mathbb{R}^{P \times r}$ and $\Theta_{r2}^{(l)} \in \mathbb{R}^{r \times P}$. This factorization, with predefined rank $r \ll P$ (with $r \le d$ in all our settings), provides a compact inter-patch communication operator with complexity linear in $P$ for fixed $r$. \textbf{Theorem~\ref{theorem:low_rank} }characterizes the approximation capacity of this factorized projection. The projected context is merged back via a residual, then refined by a feed-forward network with its own residual:
\begin{equation}
U^{(l)} = \bar{X}^{(l-1)} + \bar{H}_{\text{inter}}^{(l)},
\end{equation}
\begin{equation}
H_{\text{out}}^{(l)} = U^{(l)} + \text{FFN}^{(l)}\left(\text{Norm}^{(l)}(U^{(l)})\right),
\end{equation}
producing globally-aware representations $H_{\text{out}}^{(l)}\in\mathbb{R}^{C\times P\times d}$, where $\text{FFN}^{(l)}$ denotes the feed-forward network, and $U^{(l)}$ is the patch representation after inter-patch communication.

\subsubsection{Intra-Patch Refinement.} 
This module reshapes $H_{\text{out}}^{(l)}$ from global processing and refines each node's representation via a position-wise transform:
\begin{equation}
H_{\text{intra}}^{(l)} = \text{FeatureMixer}^{(l)}\left(\text{Norm}^{(l)}(H_{\text{out}}^{(l)})\right),
\end{equation}
where $H_{\text{intra}}^{(l)}$ holds the refined per-node features, merged back via a residual:
\begin{equation}
V^{(l)} = H_{\text{out}}^{(l)} + H_{\text{intra}}^{(l)},
\end{equation}
before a second residual applies the branch's feed-forward network:
\begin{equation}
\bar{X}^{(l)} = V^{(l)} + \text{FFN}^{(l)}\left(\text{Norm}^{(l)}(V^{(l)})\right),
\end{equation}
thereby producing layer $l$'s refined output $\bar{X}^{(l)}\in\mathbb{R}^{C\times P\times d}$ in which every node has absorbed the inter-patch context.

% BLUE-EDIT: retain the original local/intra/slot narrative while keeping the
% corrected mathematical scope of Proposition C.8.
\textbf{Approach Analysis}. Although HLI adopts a feed-forward structure, it differs fundamentally from the traditional MLP-Mixer. First, the Mixer assumes fixed-size patches, whereas Square Partition yields capacity-bounded, geometry-adaptive regions with heterogeneous node layouts; HLI therefore mixes patches at each fixed within-patch index rather than over a fixed token grid. Second, HLI decouples local and global modeling, pairing per-node position-wise refinement with a low-rank global projection motivated by the inherently low-rank nature of patch-level traffic correlations. Third, unlike the Mixer's dense all-token mixing, HLI is confined to the geometry-adaptive patch structure, reducing cost and enabling per-patch parallelism. \textbf{Theorem~\ref{theorem:low_rank}} characterizes the approximation capacity of the rank-$r$ factorization for a fixed target matrix, while \textbf{Proposition~\ref{prop:structured-mixing}} establishes the shared structured form of the isolated inter-patch projection. These differences make HLI a novel interaction operator tailored for large-scale spatio-temporal analytics.

\subsection{Output Layer}
We leverage the output $\bar{X}^{(L)}$ from the Hierarchical Linear Interaction module to predict future traffic, as illustrated in Fig.~\ref{method:modelfra}(d). We unpatch $\bar{X}^{(L)}$ by inverting the patching permutation as follows:
\begin{equation}
\tilde{Y} = \text{Unpatch}(Index, \bar{X}^{(L)}),
\end{equation}
where $\text{Unpatch}$ applies the inverse permutation of $\mathit{Index}$, and $\tilde{Y}\in\mathbb{R}^{N\times d}$ contains node representations restored to their original indexing. Finally, we design a regression convolution to predict the traffic features $Y_F$ of the following $F$ time steps:
\begin{equation}
  {Y}_F = \text{RConv}(\tilde{Y}; \Theta_r),
  \label{eq: rconv}
\end{equation}
\noindent where $\text{RConv}$ denotes regression convolution and  $\Theta_r$ are the learnable parameters for this layer.

\subsection{Topology-Aware Extension}
\label{ScalabilityAnalysis}
A potential concern with coordinate-based partitioning is that it may merge topologically disconnected links or split continuous roadways. To address this issue, the Square Partition can be extended to operate directly on road networks. %We proceed to consider to strategies for achieving this.

\subsubsection{Geometry-First Road Partitioning.}
The Square Partition algorithm can be constrained to operate on nodes sharing the same Road ID, which ensures that a patch never mixes sensors from different highways. The overall process, which replaces the original Eq.~\ref{BFS}, is formulated as follows:
\begin{equation}
\mathrm{Index} = \text{BFS}(\mathrm{Square}({(\mathrm{Lat}, \mathrm{Lng}) \mid id}, C))
\end{equation}
Here, $id$ denotes a road identifier, and ($\mathrm{Lat}$, $\mathrm{Lng}$) represents the coordinates of nodes belonging to the road identified by id. This strategy recursively partitions long roads into smaller, well-balanced segments without cross-contamination from geometrically adjacent but topologically distinct roads.

\begin{table*}[tb]
\centering
\vspace{-0.3cm}
\caption{The overall performance comparison of large-scale traffic forecasting. \redval{Bold} indicates the best performance, and \blueval{underline} denotes the second best performance. The ``-'' marker indicates unavailable results. The standard setting uses 12 historical steps to predict 12 future steps; FaST is evaluated with batch size 64. MAPE is reported in percentage.}
\vspace{-3mm}
\label{exp:main1}
\renewcommand{\arraystretch}{1}
\resizebox{0.95\textwidth}{!}{
\begin{tabular}{l|ccc|ccc|ccc|ccc}
\toprule
\multirow{2}{*}{\textbf{Method}} &
\multicolumn{3}{c|}{\textbf{SD}} &
\multicolumn{3}{c|}{\textbf{GBA}} &
\multicolumn{3}{c|}{\textbf{GLA}} &
\multicolumn{3}{c}{\textbf{CA}} \\
\cmidrule(lr){2-4}\cmidrule(lr){5-7}\cmidrule(lr){8-10}\cmidrule(lr){11-13}
 & \textbf{MAE} & \textbf{RMSE} & \textbf{MAPE} &
   \textbf{MAE} & \textbf{RMSE} & \textbf{MAPE} &
   \textbf{MAE} & \textbf{RMSE} & \textbf{MAPE} &
   \textbf{MAE} & \textbf{RMSE} & \textbf{MAPE} \\
\midrule
GWNET & 18.17 & 30.31 & 12.15 & 20.84 & 33.43 & 17.81 & 21.36 & 33.78 & 13.79 & 21.17 & 33.51 & 16.92 \\
AGCRN & 18.48 & 33.04 & 13.63 & 20.62 & 34.02 & 16.12 & 20.69 & 36.33 & 13.05 & - & - & - \\
DSTAGNN & 21.89 & 34.77 & 14.46 & 23.90 & 37.41 & 20.24 & 24.20 & 38.23 & 15.13 & - & - & - \\
D2STGNN & 17.91 & 29.59 & 11.60 & 20.78 & 33.72 & \blueval{15.11} & - & - & - & - & - & - \\
\hdashline
STID & 17.62 & \blueval{29.02} & \blueval{11.53} & 20.28 & 33.55 & 16.46 & 19.79 & 34.02 & 12.33 & 18.40 & 31.86 & 13.86 \\
BigST & 18.86 & 31.80 & 13.00 & 22.02 & 35.62 & 18.59 & 22.16 & 36.08 & 14.62 & 20.39 & 33.54 & 15.97 \\
RPMixer & 25.31 & 42.65 & 17.70 & 27.86 & 47.81 & 23.91 & 27.94 & 49.08 & 17.72 & 25.16 & 44.88 & 19.55 \\

\hdashline
STAEformer & 19.02 & 31.78 & 12.65 & 21.30 & 34.56 & 17.63 & - & - & - & - & - & - \\
STWave & 18.28 & 30.18 & 12.25 & 20.87 & 33.85 & 15.82 & 21.02 & 33.59 & 12.75 & 19.75 & 31.67 & 14.64 \\
FaST & 21.89 & 38.15 & 15.03 & 25.50 & 46.31 & 20.46 & 25.35 & 44.65 & 15.80 & 23.49 & 41.73 & 17.55 \\
PatchSTG & \blueval{17.40} & 29.55 & 12.13 & \blueval{19.67} & \blueval{33.38} & 15.14 & \blueval{18.90} & \blueval{32.09} & \blueval{11.55} & \blueval{17.75} & \blueval{30.11} & \blueval{13.77} \\
\midrule
\textbf{SqLinear} & \redval{16.48} & \redval{28.32} & \redval{10.71} & \redval{19.31} & \redval{32.77} & \redval{14.38} & \redval{18.84} & \redval{31.99} & \redval{11.29} & \redval{17.44} & \redval{29.95} & \redval{12.48} \\
\bottomrule
\end{tabular}}
\end{table*}

\subsubsection{Topology-First Road Grouping.}
When road segments are short and uniformly structured, or when topological coherence is prioritized over geometric proximity, SqLinear can bypass coordinates entirely and construct patches purely based on road affiliation. The overall process that replaces Eq.~\ref{BFS} is formulated as follows:
\begin{equation}
\mathrm{Index} = \mathrm{GroupBy}(id),
\end{equation}
where $\mathrm{GroupBy}$ assigns all nodes sharing the same road \textit{id} into one group without any further geometric subdivision. This strategy minimizes structural disruption along the same road and preserves the inherent topological continuity.

%% file: sections/5_Experiment.tex
\section{Experiments}
\label{sec:Experiments}
\setlength{\textfloatsep}{6pt plus 1pt minus 2pt}
\setlength{\floatsep}{6pt plus 1pt minus 2pt}
\setlength{\intextsep}{6pt plus 1pt minus 2pt}
\setlength{\dbltextfloatsep}{6pt plus 1pt minus 2pt}
\setlength{\dblfloatsep}{6pt plus 1pt minus 2pt}

\begin{figure*}[!tb]
    \centering
    \includegraphics[width=0.96\textwidth]{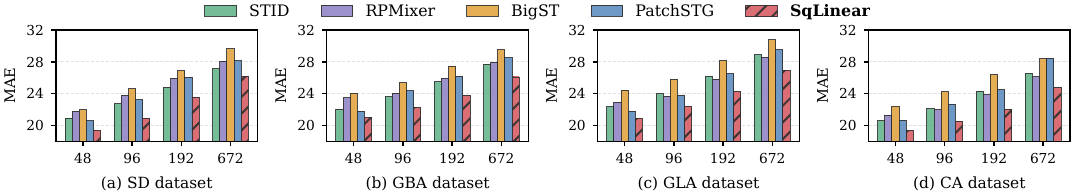}
    \caption{Performance comparison under long-horizon forecasting settings.}
    \label{fig:temporal_scalability}
\end{figure*}

We evaluate SqLinear through the following research questions:

\noindent\textbf{RQ1 (Effectiveness):} How does SqLinear perform compared with state-of-the-art baselines under standard forecasting settings?

\noindent\textbf{RQ2 (Generalization):} Can SqLinear maintain strong performance under substantially extended forecasting horizons?

\noindent\textbf{RQ3 (Efficiency):} How efficient and scalable is SqLinear in terms of computational cost and resource consumption?

\noindent\textbf{RQ4 (Ablation):} What is the contribution of each key component?

\noindent\textbf{RQ5 (Sensitivity):} How sensitive is SqLinear to different hyperparameter settings?

\noindent\textbf{RQ6 (Partition Analysis):} How do different spatial partitioning strategies affect performance?

\noindent\textbf{RQ7 (Case Study):} What insights can be obtained from the spatial dependencies learned by SqLinear?

\subsection{Experimental Setup}

We conduct extensive experiments on four large-scale traffic forecasting benchmarks. SqLinear is compared against eleven representative state-of-the-art baselines covering graph-based, transformer-based, and linear forecasting paradigms. Detailed descriptions of the \textbf{datasets}, \textbf{data preprocessing procedures}, \textbf{evaluation protocols}, \textbf{implementation details}, and \textbf{baseline configurations} are provided in \appref{app:experimental-setup}.

\subsection{Overall Performance Comparison (RQ1)}
\label{Main Results}
%\textbf{Overall Performance.} 
Table~\ref{exp:main1} reports the forecasting results under the standard 12-step setting. \textbf{SqLinear consistently achieves the best performance across all datasets and evaluation metrics}, demonstrating that its scalable design does not compromise predictive accuracy. Compared with the strongest competing method on each dataset, SqLinear reduces MAE by an average of 2.30\% and MAPE by 5.89\%. Specifically, we have the following main observations.

\textbf{(i) Comparison with Existing Forecasting Paradigms.} Table~\ref{exp:main1} reports three paradigms: graph-based models, attention-based models, and linear architectures. While graph-based methods (e.g., D2STGNN and DSTAGNN) achieve competitive accuracy by explicitly modeling sensor interactions, several of them fail to scale to the larger datasets due to the high cost of graph propagation. Attention-based methods improve representation capacity but often incur substantial computational overhead as the number of sensors increases. Linear models, such as STID and BigST, offer superior efficiency, yet their simplified architectures may limit their ability to capture complex spatial dependencies. \textbf{SqLinear combines the strengths of these paradigms by employing scalable spatial partitioning and hierarchical linear interaction, achieving both superior forecasting accuracy and scalability}.

\textbf{(ii) Benefits of Spatial Partitioning.} A notable observation is that \textbf{methods incorporating spatial patching outperform conventional graph- or sequence-based architectures} on large-scale datasets. This trend suggests that appropriately organizing spatial sensors into structured local regions is crucial for large-scale traffic forecasting. By generating balanced and compact patches, the proposed Square Partition preserves local spatial correlations while avoiding the irregular structures and excessive padding introduced by existing partitioning strategies. This provides a more effective foundation for subsequent spatio-temporal modeling.

\textbf{(iii) Benefits of Hierarchical Linear Interaction.} Among all methods, PatchSTG achieves the most competitive results, highlighting the effectiveness of patch-level spatial modeling. Nevertheless, \textbf{SqLinear consistently outperforms PatchSTG}. This improvement demonstrates that, given a well-structured spatial partition, hierarchical linear interactions are sufficient to propagate long-range spatial dependencies without relying on expensive attention operations. The results suggest that the combination of structured partitioning and lightweight interaction mechanisms offers a promising direction for scalable traffic forecasting.

%\subsection{Scalability Study (RQ2)}

\subsection{Long-Horizon Generalization (RQ2)}
\label{sec:temporal_scalability}

To evaluate the robustness of SqLinear under extended forecasting horizons, we increase the prediction length from the standard 12-step setting to 48, 96, 192, and 672 steps while using 96 historical observations as input. Fig.~\ref{fig:temporal_scalability} compares SqLinear with representative scalability-oriented baselines across four large-scale datasets. As observed, SqLinear consistently achieves the lowest MAE in all 16 dataset--horizon combinations. Compared with the strongest competing method in each setting, SqLinear reduces MAE by an average of 3.17\% and by up to 6.78\%. Due to space limitations, detailed results in terms of MAE, RMSE, and MAPE are provided in \appref{app:scalability-metrics}.

A notable observation is that the performance gap generally remains stable as the forecasting horizon increases. This result suggests that \textbf{SqLinear does not merely fit short-term traffic dynamics but can effectively capture long-range spatio-temporal dependencies}. We attribute this robustness to two key factors. First, Square Partition organizes sensors into compact and balanced spatial regions, providing a stable spatial representation that remains effective across different prediction horizons. Second, the hierarchical linear interaction mechanism combines global inter-patch communication with per-node local refinement, enabling information propagation across multiple spatial scales without introducing the optimization challenges and computational overhead of attention-based architectures. Overall, the results demonstrate that SqLinear generalizes effectively from short- to long-horizon forecasting scenarios, making it suitable for large-scale traffic forecasting tasks requiring both accuracy and temporal robustness.

\subsection{Efficiency and Scalability Evaluation (RQ3)}
\label{sec:efficiency}

\begin{figure}[!tbp]
\captionsetup[subfigure]{labelfont=normal, textfont=normal}
    \centering
    \begin{subfigure}{0.45\linewidth}
        \includegraphics[width=\linewidth]{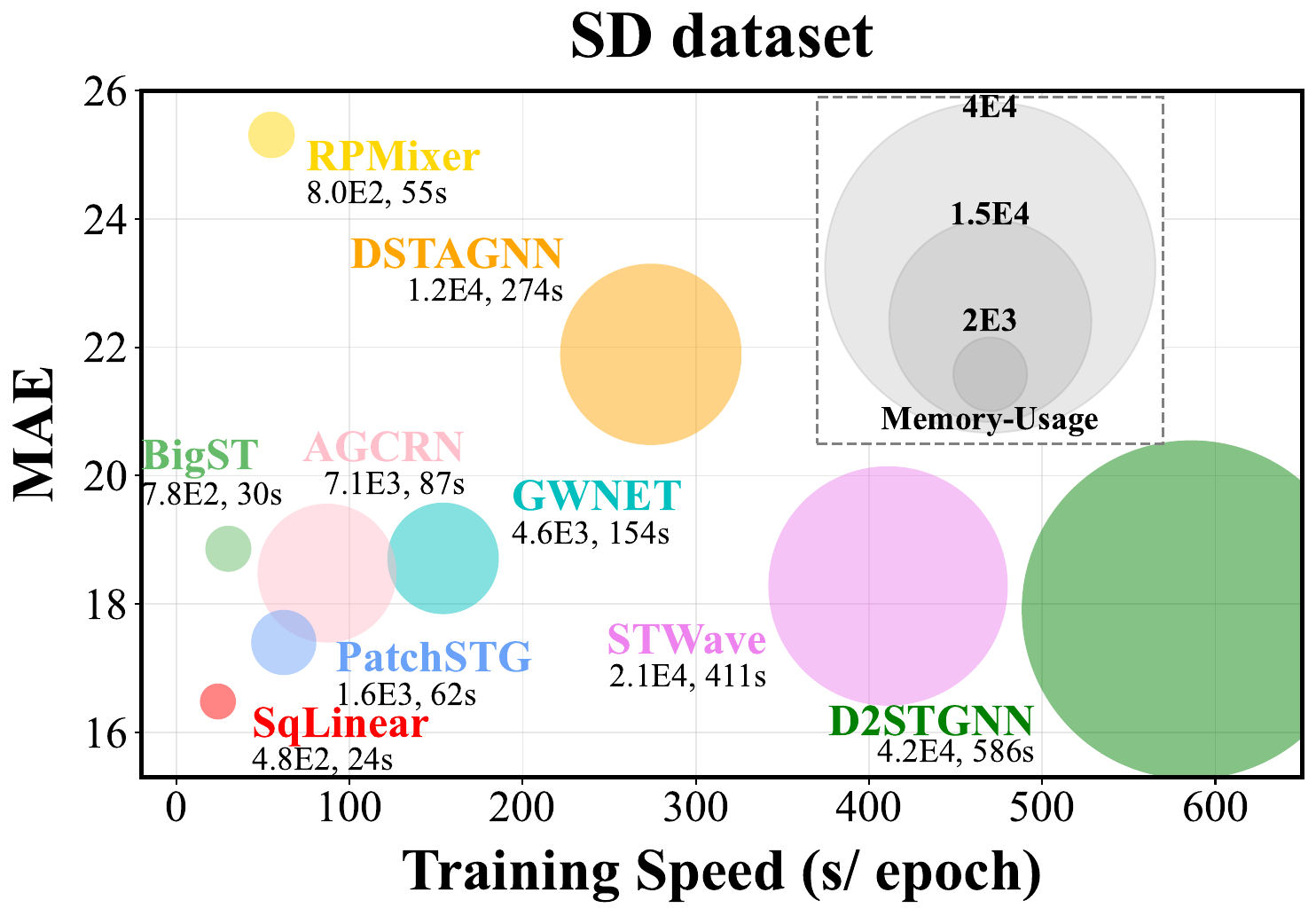}
        \caption{SD dataset}
    \end{subfigure}
    \hfill
    \begin{subfigure}{0.45\linewidth}
        \includegraphics[width=\linewidth]{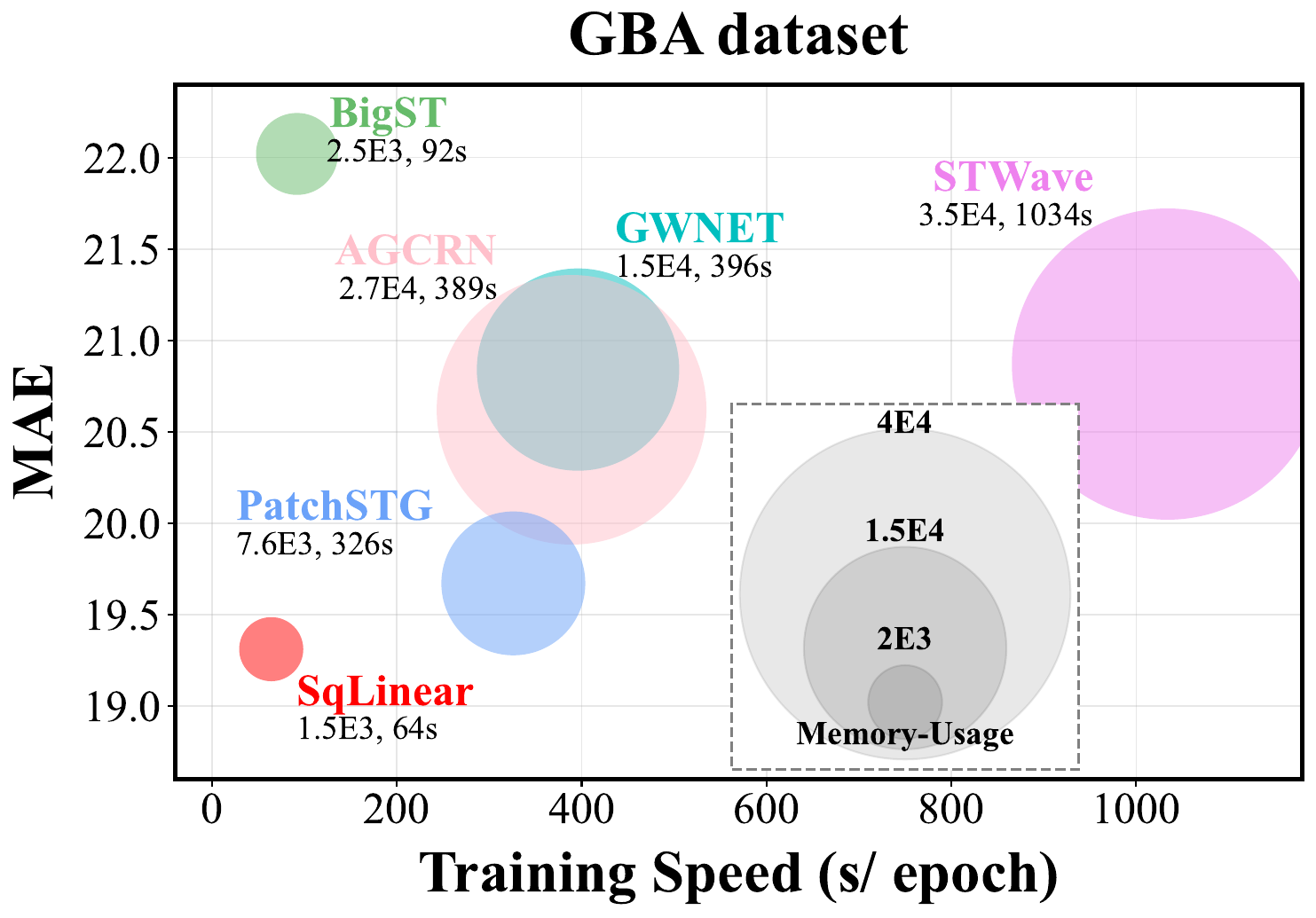}
        \caption{GBA dataset}
    \end{subfigure}
    \vspace{-0.3cm}
    \caption{Accuracy-efficiency bubble plots under the standard setting. The x-axis reports training runtime, the y-axis reports MAE, and bubble size indicates peak GPU memory.}
    \label{fig:bubble_diagram}
\end{figure}

%\textbf{Accuracy-Efficiency Trade-off.} %We examine efficiency through the accuracy-efficiency trade-off and practical resource cost. Under the standard setting, Fig.~\ref{fig:bubble_diagram} summarizes forecasting accuracy, training runtime, and peak GPU memory on SD and GBA. SqLinear lies in a favorable region of the bubble plots: it achieves low MAE with low runtime and modest memory consumption. This supports the claim that its accuracy gains are not obtained by simply increasing computational cost.
%We first evaluate whether the accuracy gains of SqLinear come at the expense of increased computational cost. 
Fig.~\ref{fig:bubble_diagram} summarizes accuracy, runtime, and peak GPU memory under the standard setting. As observed, SqLinear consistently occupies a favorable region of the accuracy-efficiency spectrum, achieving the lowest forecasting error while maintaining low runtime and moderate memory consumption. These results indicate that the performance improvements of SqLinear are not obtained through excessive model complexity or resource usage.

\begin{figure}[!tbp]
    \centering
    \includegraphics[width=0.9\linewidth]{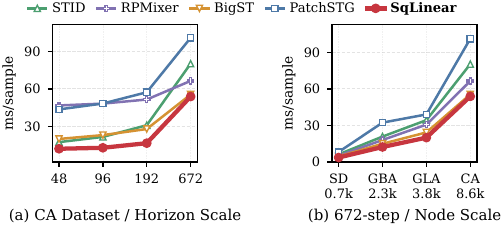}
    \vspace{-0.1cm}
    \caption{Per-sample training time comparison.}
    \label{fig:temporal_train_speed_lines}
\end{figure}

%\textcolor{blue}{\textbf{Speed Scaling.} Fig.~\ref{fig:temporal_train_speed_lines} compares long-horizon training efficiency using per-sample training time. On the largest CA dataset, SqLinear achieves the lowest training time at all four prediction lengths and reduces training time by 46.78\% on average relative to the most efficient baseline in each setting. Under the longest 672-step horizon, SqLinear again ranks first on every dataset and reduces training time by 17.04\% on average. These results show that SqLinear maintains state-of-the-art training efficiency when scaling along both prediction horizon and sensor-node count.}

\textbf{Scaling Behavior.} We further investigate how the computational cost of SqLinear evolves when scaling along two key dimensions: forecasting horizon and sensor-node count. Fig.~\ref{fig:temporal_train_speed_lines} reports the per-sample training time under long-horizon forecasting settings. On the largest CA dataset, SqLinear achieves the shortest per-sample training time across all prediction horizons and reduces per-sample training time by 46.78\% on average compared with the most efficient baseline in each setting. Under the most challenging 672-step forecasting scenario, SqLinear consistently ranks first across all datasets and achieves an average speedup of 17.04\%. These results demonstrate that the proposed Square Partition and Hierarchical Linear Interaction effectively control computational growth as both temporal and spatial scales increase. In particular, the partition-based design avoids expensive dense interactions among all sensor nodes, enabling efficient information aggregation while maintaining strong predictive performance.

\begin{table}[!tbp]
\centering
\caption{\label{Table:Efficiency_horizon}Computational cost on the largest CA dataset.}
\setlength{\tabcolsep}{2pt}
\resizebox{0.95\linewidth}{!}{
\begin{tabular}{l|ccc|ccc|ccc|ccc}
\toprule
\multirow{2}{*}{\textbf{Method}} &
    \multicolumn{3}{c|}{\textbf{48-step}} &
    \multicolumn{3}{c|}{\textbf{96-step}} &
    \multicolumn{3}{c|}{\textbf{192-step}} &
    \multicolumn{3}{c}{\textbf{672-step}} \\
\cmidrule(lr){2-4}\cmidrule(lr){5-7}\cmidrule(lr){8-10}\cmidrule(lr){11-13}
& B. & G. & P. & B. & G. & P. & B. & G. & P. & B. & G. & P. \\
\midrule
STID & 64 & 7.0 & 0.77 & 64 & 7.8 & 0.78 & 64 & 9.9 & 0.79 & 64 & 22.2 & 0.87 \\
BigST & 64 & 20.7 & 0.36 & 64 & 20.8 & 0.37 & 64 & 21.0 & 0.39 & 64 & 31.5 & 0.52 \\
RPMixer & 64 & 23.8 & 7.81 & 64 & 24.4 & 7.83 & 64 & 25.6 & 7.85 & 64 & 36.5 & 7.99 \\
FaST & 64 & 14.6 & 1.31 & 64 & 15.5 & 1.32 & 64 & 17.4 & 1.34 & 64 & 28.7 & 1.47 \\
PatchSTG & 32 & 26.9 & 1.85 & 32 & 27.3 & 1.86 & 32 & 28.3 & 1.87 & 32 & 34.5 & 1.95 \\
\midrule
\textbf{SqLinear} & 64 & 25.2 & 1.23 & 64 & 25.8 & 1.24 & 64 & 27.7 & 1.25 & 64 & 31.4 & 1.06 \\
\bottomrule
\end{tabular}}
\end{table}

\begin{table}[!tbp]
\centering
\caption{\label{Table:Efficiency_dataset}Computational cost for longest 672-step prediction.}
\setlength{\tabcolsep}{2pt}
\resizebox{0.95\linewidth}{!}{
\begin{tabular}{l|ccc|ccc|ccc|ccc}
\toprule
\multirow{2}{*}{\textbf{Method}} &
\multicolumn{3}{c|}{\textbf{SD}} &
\multicolumn{3}{c|}{\textbf{GBA}} &
\multicolumn{3}{c|}{\textbf{GLA}} &
\multicolumn{3}{c}{\textbf{CA}} \\
\cmidrule(lr){2-4}\cmidrule(lr){5-7}\cmidrule(lr){8-10}\cmidrule(lr){11-13}
& B. & G. & P. & B. & G. & P. & B. & G. & P. & B. & G. & P. \\
\midrule
GWNET & 32 & 28.8 & 0.89 & - & - & - & - & - & - & - & - & - \\
AGCRN & 32 & 38.9 & 0.80 & - & - & - & - & - & - & - & - & - \\
DSTAGNN & 16 & 23.5 & 6.22 & - & - & - & - & - & - & - & - & - \\
\hdashline
STID & 64 & 1.9 & 0.37 & 64 & 6.1 & 0.47 & 64 & 9.9 & 0.57 & 64 & 22.2 & 0.87 \\
BigST & 64 & 2.7 & 0.26 & 64 & 8.8 & 0.32 & 64 & 14.4 & 0.36 & 64 & 31.5 & 0.52 \\
RPMixer & 64 & 3.1 & 1.69 & 64 & 10.0 & 2.47 & 64 & 16.3 & 3.46 & 64 & 36.5 & 7.99 \\

\hdashline
STAEformer & 4 & 28.4 & 16.34 & - & - & - & - & - & - & - & - & - \\
FaST & 64 & 3.4 & 2.08 & 64 & 11.0 & 2.35 & 64 & 12.8 & 1.01 & 64 & 28.7 & 1.47 \\
PatchSTG & 64 & 6.6 & 2.42 & 64 & 28.7 & 3.29 & 64 & 25.1 & 1.80 & 32 & 34.5 & 1.95 \\
\midrule
\textbf{SqLinear} & 64 & 2.4 & 0.59 & 64 & 9.1 & 0.71 & 64 & 15.0 & 0.77 & 64 & 31.4 & 1.06 \\
\bottomrule
\end{tabular}}
\end{table}

%\textcolor{blue}{\textbf{Resource Cost.} Tables~\ref{Table:Efficiency_horizon} and~\ref{Table:Efficiency_dataset} summarize the resource consumption of the same long-horizon models using separated cost fields, where B., G., and P. denote batch size, GPU memory usage in GB, and the number of parameters in M, respectively; the ``-'' marker indicates out-of-memory results. This layout makes the trade-offs easier to inspect: SqLinear is not always the smallest model in parameter count, but its parameters stay at or below 1.25M across all reported long-horizon settings. Its memory footprint remains controlled under the most demanding setting: at 672-step forecasting on CA, SqLinear uses 31.4GB, lower than PatchSTG and RPMixer and close to BigST while using a larger batch size than PatchSTG. Together with the bubble plots and runtime trends, these measurements show that SqLinear offers a favorable accuracy-efficiency trade-off under both standard and extended temporal settings.}

\textbf{Resource Consumption.} Tables~\ref{Table:Efficiency_horizon} and~\ref{Table:Efficiency_dataset} provide a detailed comparison of resource consumption under long-horizon forecasting settings, where B., G., and P. denote batch size, GPU memory usage (GB), and parameter count (M), respectively. The ``-'' marker indicates out-of-memory results. Although SqLinear is not always the smallest model in terms of parameter count, its model size remains compact, requiring no more than 1.25M parameters across all evaluated settings. More importantly, its memory consumption remains well controlled under large-scale scenarios. For example, under 672-step forecasting on the CA dataset, SqLinear requires 31.4GB of GPU memory, which is lower than PatchSTG and RPMixer, while supporting a larger batch size than PatchSTG. This observation suggests that the proposed architecture utilizes computational resources more effectively than competing approaches.

Overall, the runtime, memory, and scalability analyses consistently show that SqLinear achieves a favorable accuracy-efficiency trade-off. The results confirm that the proposed partition-based linear framework scales gracefully to large traffic networks and long forecasting horizons without sacrificing predictive accuracy.

\subsection{Ablation Study (RQ4)}
\label{sec:ablation}
We perform ablation studies on the GBA and CA datasets. Similar results were observed on other datasets. We assess the contribution of key design choices by comparing the full SqLinear model with eight variants: (1)~\textbf{w/ Grid}, (2) \textbf{w/ Quadtree}, and (3) \textbf{w/ K-D Tree} replace square partitioning with Grid, Quadtree, and K-D Tree structures (using leaf padding); (4)~\textbf{w/ AlterAxis} uses alternating axes instead of Eq.~\ref{splitaxis}; (5) \textbf{w/ MedianSplit} uses the median split point instead of Eq.~\ref{splitpoint}; (6)~\textbf{w/o Inter} and (7) \textbf{w/o Intra} remove the inter-patch and intra-patch branches, respectively; (8)~\textbf{w/ Attn} replaces the FeatureMixer with standard attention. As shown in Fig.~\ref{exp:Ablation}, we make the following key observations.

\begin{figure}[!tb]
\captionsetup[subfigure]{labelfont=normal, textfont=normal}
    \centering
    \begin{subfigure}{0.42\linewidth}
        \includegraphics[width=\linewidth]{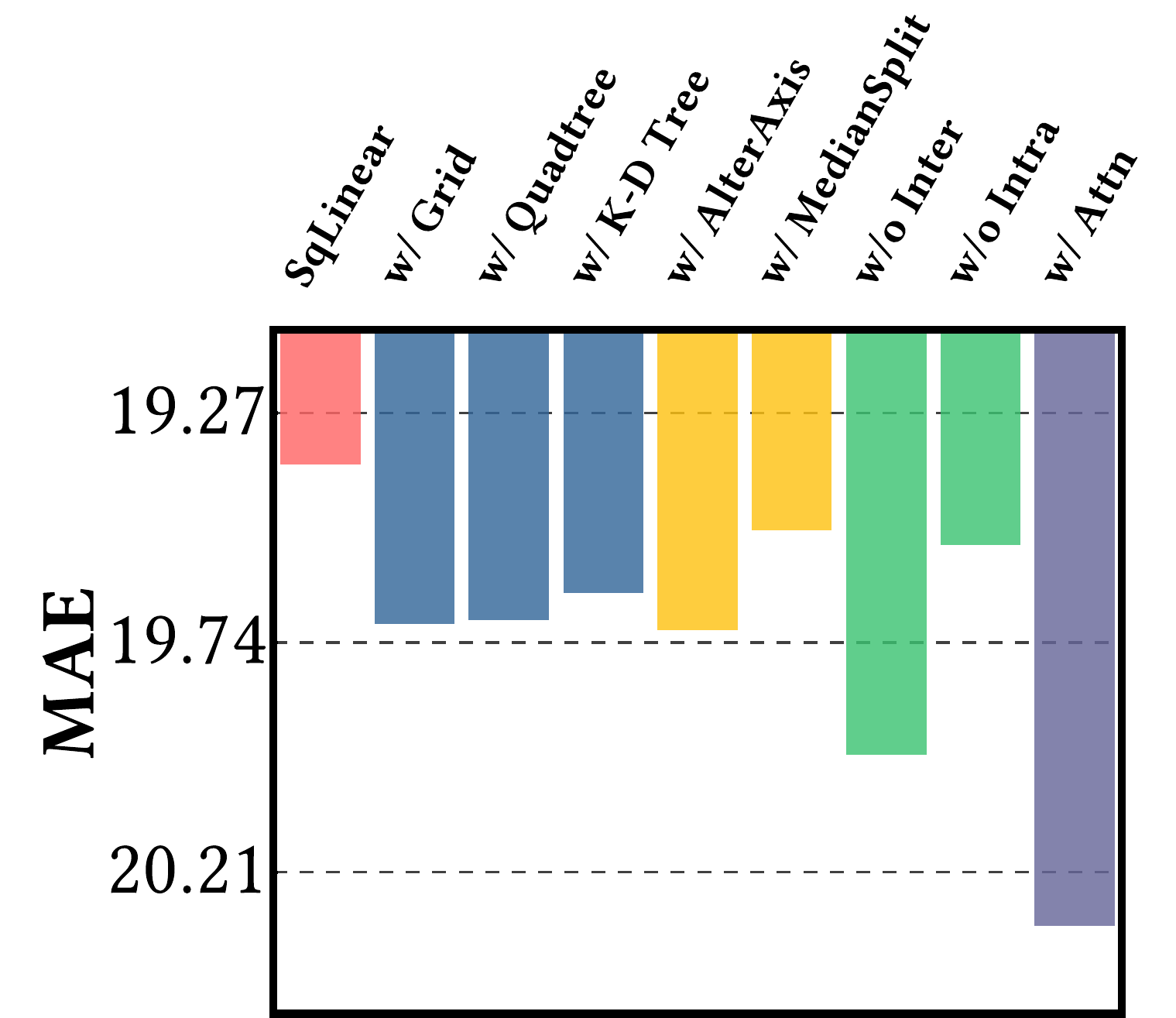}
        \caption{GBA dataset}
    \end{subfigure}
    \hfill
    \begin{subfigure}{0.42\linewidth}
        \includegraphics[width=\linewidth]{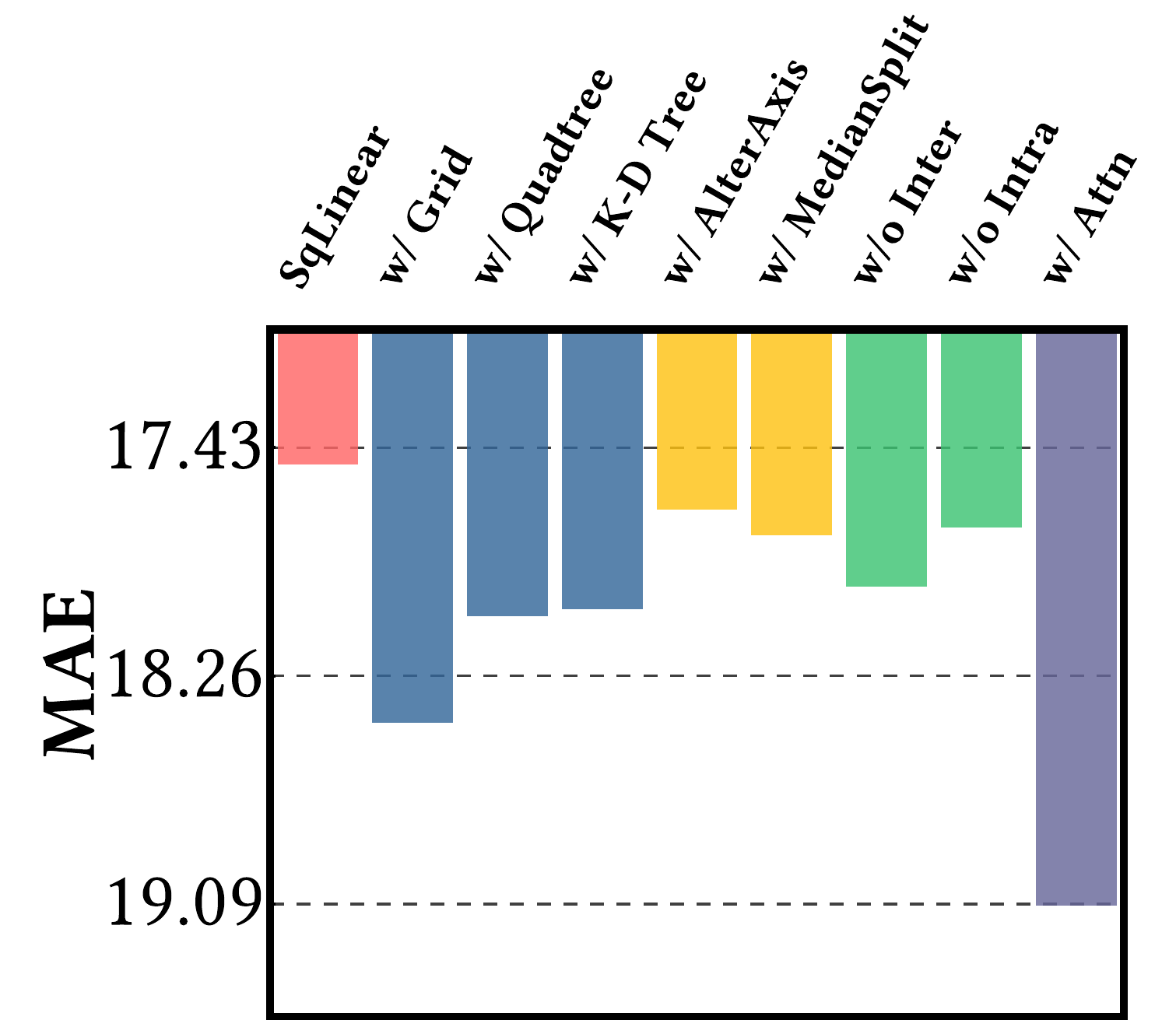}
        \caption{CA dataset}
    \end{subfigure}
    \vspace{-0.3cm}
    \caption{Ablation study.}
    \label{exp:Ablation}
    \vspace{-0.1cm}
\end{figure}

\textbf{Effectiveness of Square Partition}. For a fair comparison, the K-D Tree variant adopts the same leaf capacity $C$ as Square Partition. Since Grid and Quadtree do not have an explicit leaf capacity, we tune their grid resolution and tree depth such that the number of samples per leaf cell is approximately comparable to $C$. Experimental results validate our partitioning strategy from multiple perspectives. First, conventional spatial partitioning methods, including Grid, Quadtree, and K-D Tree, lead to substantial performance degradation. This is mainly due to the generation of imbalanced or geometrically irregular patches, where excessive padding with artificial entries may dilute feature representations. Moreover, the inferior performance of AlterAxis and MedianSplit further demonstrates the importance of our design choices in selecting both the splitting axis and split point, which are crucial for producing balanced patch indices and enabling padding-free tensorization under a fixed capacity constraint.

\textbf{Effectiveness of Hierarchical Linear Interaction}. Performance drops observed in w/o Inter and w/o Intra demonstrate the necessity of both branches: removing inter-patch communication cuts cross-patch information exchange, while removing intra-patch refinement removes the nonlinearity that integrates this exchange into each node. In addition, the w/ Attn variant not only fails to outperform HLI but also introduces additional computational overhead. These results suggest that, given well-structured partitioned inputs, a lightweight linear architecture is sufficient for modeling spatio-temporal dependencies, whereas more complex attention-based designs may incur unnecessary computational cost without performance gains.

\subsection{Additional Experiments (RQ5, RQ6, RQ7)}
\label{sec:additional_analyses}

Further analyses are deferred to the appendix, where we study hyperparameters in \appref{app:Hyper-parameter}, visualize and compare spatial partitioning strategies in \appref{sec:partition_strategy}, and present a case study of learned spatial dependencies in \appref{app:case-study}.

%% file: sections/6_Conclusion.tex
\section{Conclusion}
\label{sec:CONCLUSION}
We presented SqLinear, a novel linear architecture designed to improve the accuracy and efficiency of traffic forecasting. By employing adaptive square partitioning with hierarchical linear interaction, SqLinear captures complex spatiotemporal dependencies while avoiding the quadratic overhead of existing methods.

%% file: Appendix/4_RelatedWork.tex
\section{Related Work}
\label{sec:Related Work}
We cover related studies on spatio-temporal traffic prediction and commonly used patching strategies in deep learning.

\subsection{Spatio-Temporal Prediction}

Spatio-temporal traffic prediction aims to forecast traffic flow, speed, and occupancy. Early approaches primarily relied on statistical models, including moving average models~\cite{durbin1959efficient}, AutoRegressive Integrated Moving Average (ARIMA)~\cite{box1970distribution}, and multivariate regression~\cite{alexopoulos2010introduction}. To model non-linear and complex spatio-temporal dependencies, numerous deep learning-based spatio-temporal prediction methods~\cite{yao2025leveraging,luo2025towards,kudrat2025patch,zhao2026fast,zheng2025st} have emerged. 

We proceed to highlight the most relevant state-of-the-art methods, while more categories of existing spatio-temporal prediction models are covered by surveys~\cite{wang2020deep,jin2023spatio,wang2024deep}. In general, these methods employ sequence-oriented architectures (e.g., RNNs, LSTMs) to capture temporal correlations and convolutional architectures (e.g., CNNs, GNNs) to capture spatial dependencies. Spatio-Temporal Graph Neural Networks (STGNNs) represent the dominant paradigm, combining graph-based spatial modeling with temporal processing modules. These networks employ architectures that generally fall into two categories: recurrent-based approaches, such as DCRNN~\cite{DBLP:conf/iclr/LiYS018} and DGCRN~\cite{li2023dynamic}, that use RNN variants to model temporal dynamics, and convolutional-based methods, such as GWNET~\cite{wu2019graph} and AGCRN~\cite{bai2020adaptive}, that use temporal convolutional networks. 

Inspired by the success of the Transformer architecture~\cite{vaswani2017attention}, researchers have explored attention-based spatio-temporal prediction models, such as ASTGCN~\cite{guo2019attention} and STWave~\cite{fang2023spatio}, to improve prediction accuracy. These improvements come at high computational costs, with graph convolution operations typically having quadratic complexity in terms of node count, and with the attention mechanisms introducing high complexity. 

To improve model scalability, simplified architectures have been explored~\cite{li2025unifying,li2026minitraffic}. MLP-based models such as STID~\cite{shao2022spatial} and RPMixer~\cite{yeh2024rpmixer}, demonstrate that carefully designed multilayer perceptrons can achieve competitive performance, while linear architectures like BigST~\cite{han2024bigst} and GSNet~\cite{GraphSparseNet} show that properly normalized linear projections can capture spatio-temporal patterns effectively. FaST~\cite{zhao2026fast} further targets efficient long-horizon forecasting on large-scale spatial-temporal graphs through a mixture-of-experts architecture with adaptive graph-agent attention. However, these efficiency gains often come at the expense of explicit spatial reasoning, particularly for modeling non-local dependencies that are critical in transportation networks. %\textit{Unlike previous studies that primarily focus on improving model accuracy, we propose a new architecture that preserves the spatial awareness of graph-based methods while achieving the efficiency of streamlined prediction models.}
\textbf{\textit{As shown in Table~\ref{Related:comp}, compared with previous efficiency- or scalability-oriented methods, the proposed SqLinear method supports large-scale inputs, provides spatial management, and includes scalability-oriented theoretical analysis while achieving linear complexity.}}

The success of large language models (LLMs) has inspired another line of studies. Foundation models such as Unist~\cite{10.1145/3637528.3671662} and UrbanGPT~\cite{10.1145/3637528.3671578} attempt to adapt LLM architectures to spatio-temporal settings, while hybrid systems like ST-LLM+~\cite{11005661} combine traditional forecasting modules with Transformer backbones. As covered comprehensively~\cite{fang2025unravelingspatiotemporalfoundationmodels}, these approaches represent promising but preliminary efforts toward general-purpose spatio-temporal learning, which is beyond the scope of this paper.

\subsection{Patching Strategies in Deep Learning}

The partitioning of complex and massive data sets has emerged as a fundamental paradigm for enhancing computational efficiency of machine learning models~\cite{park2022vision,xie2021self,zhang2022patchformer,liang2024foundation}. Patching strategies vary across data modalities~\cite{zhong2025mtm}. For regularly structured data like images and video, fixed-size spatial aggregation dominates, as exemplified by Vision Transformers~\cite{dosovitskiy2020image} that decompose images into uniform patches and by Swin Transformers~\cite{liu2021swin} that employ hierarchical patch merging. 

The methods PatchTST~\cite{nietime} and Pathformer~\cite{chenpathformer} develop patch-based strategies for time series forecasting. However, these methods patch only in the time dimension and lack spatial modeling capabilities. Further, they are designed for small-scale time series data as opposed to large-scale spatio-temporal data processing. In contrast, irregular spatial data, particularly from traffic sensor networks, present unique spatio-temporal challenges due to the heterogeneous spatial layouts and skewed node distributions of the networks~\cite{xie2021statistically}. LISA~\cite{an2024lisa} introduces a partitioning framework that iteratively learns sub-regions based on prediction accuracy. However, its expansion-based process is computationally intensive and is primarily optimized for sparse accident prediction scenarios, which limits its applicability to general traffic forecasting tasks. To contend with general irregular graphs, PatchGT~\cite{gao2022patchgt} utilizes non-trainable spectral clustering to segment graphs. However, this approach often fails to adapt to dynamic data distributions and evolving spatial patterns. More recent methods such as PatchSTG~\cite{fang2024efficient} adopt leaf-level K-D Tree partitioning to balance node distributions across partitions. While effective at mitigating node imbalance, K-D Trees come with inherent intrinsic geometric limitations, including unavoidable patch elongation and mandatory padding operations. Such artifacts introduce additional computational overhead and spatial distortion, as illustrated in Fig.~\ref{intro:model}. In contrast to existing studies~\cite{fang2024efficient,10.1145/355744.355745,moti2022waffle}, Square Partition is tailored to traffic forecasting rather than general spatial indexing: it (i) uses capacity-aware split positions to control leaf sizes, (ii) uses coordinate-adaptive axis selection to encourage compact spatial patches, and (iii) produces padding-free patch representations for downstream spatio-temporal modeling.

%% file: Appendix/1_Algorithm.tex
\section{Algorithm}
\label{app:alg}

Algorithm~\ref{alg:square-partition} details the {Square Partition} algorithm. The algorithm begins by checking whether the input node set $N^l$ satisfies the base case, i.e., whether its size does not exceed a predefined capacity threshold $C$ (lines 2--4). If the base case is not satisfied, the algorithm selects the split axis (latitude or longitude) by comparing the spatial spans along each dimension (lines 5--11). It then calculates a capacity-aware split point for balanced recursive subdivision (line 13). Based on the split point, the node set $N^l$ is partitioned into two subsets, $N^{l+1}_1$ and $N^{l+1}_2$, and the algorithm is applied recursively to each subset. The final output is the set of the recursively partitioned results (lines 14--18).

\begin{algorithm}[H]
\caption{The Square Partition Algorithm}
\label{alg:square-partition}
\begin{algorithmic}[1]
\Require{
    nodeset $N^l$: nodes with coordinates.
    LeafCapacity $C$: Number of nodes per leaf patch.
}
\Ensure{
    $P$: a set of patches.
}
\Function{SquarePartition}{$N^l$}
    \If{$|N^l| \le C$} \Comment{Node count is within capacity}
        \State \textbf{return} $N^l$
    \EndIf

    \Statex \Comment{\textit{\underline{Step 1}: Determine splitting axis}}
    \State $\Delta_{Lat} \gets \max_{n \in N^l}(Lat_{n}) - \min_{n \in N^l}(Lat_{n})$
    \State $\Delta_{Lng} \gets \max_{n \in N^l}(Lng_{n}) - \min_{n \in N^l}(Lng_{n})$
    \If{$\Delta_{Lng} \ge \Delta_{Lat}$} \Comment{Calculate SplitAxis}
        \State $SplitAxis \gets \text{Longitude}$
    \Else
        \State $SplitAxis \gets \text{Latitude}$
    \EndIf

    \State Sort $N^l$ based on the coordinates of $SplitAxis$.

    \Statex \Comment{\textit{\underline{Step 2}: Identify splitting position}}
    \State {$S_p \gets C\left\lceil\dfrac{|N^l|-C}{2C}\right\rceil$} \Comment{{Calculate split point}}

    \Statex \Comment{\textit{\underline{Step 3}: Partition and recurse}}
    \State $N^{l+1}_1 \gets N^l[1 \dots S_p]$ \Comment{$1st$ subset}
    \State $N^{l+1}_2 \gets N^l[S_p+1 \dots |N^l|]$ \Comment{$2nd$ subset}

    \State $P_1 \gets \Call{SquarePartition}{N^{l+1}_1}$ \Comment{Partition $1st$ set}
    \State $P_2 \gets \Call{SquarePartition}{N^{l+1}_2}$ \Comment{Partition $2nd$ set}

    \State \textbf{return} $P_1 \cup P_2$
\EndFunction
\end{algorithmic}
\end{algorithm}

Algorithm~\ref{alg:sqlinear} outlines the complete workflow of {SqLinear}, from initial data embedding to final prediction. The process begins with the {Spatio-Temporal Embedding Layer}, where a unified feature representation $\tilde{X}$ is constructed by generating and concatenating the historical embeddings $E_H$, temporal embeddings $E_T^d$ and $E_T^w$, and spatial embedding $E_S$. SquarePartition first produces node indices according to sensor coordinates and capacity $C$, and the Patching operation then groups the node representations in $\tilde{X}$ into the patch tensor $\bar{X}$ without modifying feature dimensions. At the core of SqLinear, the {Hierarchical Linear Interaction} block processes the data iteratively across $L$ layers. In each layer $l$, an {{Inter-Patch Communication}} module first processes $\bar{X}^{(l-1)}$ to capture dependencies across patches, resulting in $H_{\text{out}}^{(l)}$. Then, the {{Intra-Patch Refinement}} module {refines each node's representation with this context}, updating the patch representations to $\bar{X}^{(l)}$. Finally, the {Output Layer} reconstructs the full spatio-temporal graph $\tilde{Y}$ from the final-layer patch representations $\bar{X}^{(L)}$ via the {Unpatch} operation, and applies reverse convolution ({RConv}) to obtain the final prediction $Y_F$.

\begin{algorithm}[H]
\caption{The SqLinear Algorithm}
\label{alg:sqlinear}
\begin{algorithmic}[1]
\Require{input data $X_H$, nodeset $N^{0}$, leaf capacity $C$, number of layers $L$}
\Ensure{predicted traffic features $Y_F$}
\Statex \Comment{\textit{\underline{Spatio-Temporal Embedding Layer}}}
\State $E_H \gets \mathit{PConv}({[X_H \vert\vert \tau_d/T_d \vert\vert \tau_w/T_w]}; \Theta_h)$ \Comment{Historical embedding}
\State $E_T^d \gets W_{day} {[\tau_d]}$ \Comment{{Time}-of-day embeddings}
\State $E_T^w \gets W_{week}{[\tau_w]}$ \Comment{Day-of-week embedding}
\State $E_S \gets {W_s}$ \Comment{Spatial embedding}
\State $\tilde{X} \gets E_H \vert\vert E_T^{d} \vert\vert E_T^{w} \vert\vert E_S$ \Comment{Combine all embeddings}
\Statex \Comment{\textit{\underline{Spatial Partition}}}
\State $\mathit{Index} \gets \text{SquarePartition}(N^{0}, C)$ \Comment{Generate node grouping indices}
\State $\bar{X} \gets \text{Patching}(\mathit{Index}, \tilde{X})$ \Comment{Group node representations}
\Statex \Comment{\textit{\underline{Hierarchical Linear Interaction}}}
\For{$l \gets 1$ \textbf{to} $L$}
    \Statex \Comment{{Inter-Patch Communication}}
    \State $H_{\text{inter}}^{(l)} \gets {\text{FeatureMixer}}^{(l)}(\text{Norm}^{(l)}(\bar{X}^{(l-1)}))$
    \State $\bar{H}_{\text{inter}}^{(l)} \gets \text{LowRankProj}^{(l)}(H_{\text{inter}}^{(l)})$
    \State {$U^{(l)} \gets \bar{X}^{(l-1)} + \bar{H}_{\text{inter}}^{(l)}$}
    \State $H_{\text{out}}^{(l)} \gets {U^{(l)}} + \text{FFN}^{(l)}({\text{Norm}^{(l)}(U^{(l)})})$
    \Statex \Comment{{Intra-Patch Refinement}}
    \State $H_{\text{intra}}^{(l)} \gets {\text{FeatureMixer}}^{(l)}(\text{Norm}^{(l)}(H_{\text{out}}^{(l)}))$
    \State {$V^{(l)} \gets H_{\text{out}}^{(l)} + H_{\text{intra}}^{(l)}$}
    \State $\bar{X}^{(l)} \gets {V^{(l)}} + \text{FFN}^{(l)}({\text{Norm}^{(l)}(V^{(l)})})$
\EndFor
\Statex \Comment{\textit{\underline{Output Layer}}}
\State $\tilde{Y} \gets \text{Unpatch}({\mathit{Index}, }\bar{X}^{(L)})$ \Comment{Reconstruct from patches}
\State $Y_F \gets \text{RConv}(\tilde{Y}; \Theta_r)$ \Comment{Generate final prediction}
\State \textbf{return} $Y_F$
\end{algorithmic}
\end{algorithm}

%% file: Appendix/2_Theory.tex
\section{Theoretical Analysis}
\label{app:theory}

We provide a concise analysis of the properties that directly support scalable traffic forecasting. In the Square Partition approach, the key theoretical function is to generate capacity-bounded tensor units that achieve high utilization and controlled split imbalance, whereas compactness is treated as a distribution-dependent property—encouraged by the longest-span splitting rule and validated empirically. We then analyze the approximation capacity of the low-rank inter-patch projection, and conclude with the original complexity analysis.
\subsection{Analysis of Spatial Partition}
\label{subsubsec:partition_analysis}
We analyze Square Partition from the perspective of downstream tensor processing. A good spatial partition for SqLinear should (i) avoid padded or empty tokens, (ii) keep recursive splits balanced to avoid skewed patch workloads, and (iii) encourage compact spatial neighborhoods.

\begin{definition}[\textbf{Patch Utilization}]
\label{def:Patch Utilization}
{
Let $P$ be the number of generated leaf patches, let $C$ be the patch capacity, and let $N$ be the total number of sensor nodes. We define patch utilization as
\begin{equation}
\eta = \frac{N}{P \times C},
\end{equation}
where larger $\eta$ indicates fewer padded or unused tensor entries during patch-based modeling.
}
\end{definition}

\begin{theorem}[\textbf{Padding-Free Tensorization}]
\label{theorem:Patch Utilization}
{
{If $C$ divides $N$, Square Partition produces exactly $P=N/C$ leaf patches, each containing exactly $C$ nodes. Hence the patch tensor has no padded sensor entries and satisfies $\eta=1$. Otherwise, at most one leaf patch is under-filled, so at most $C-1$ padded entries are required in total.}
}
\end{theorem}

\textit{Proof sketch.} {{By Eq.~\ref{splitpoint}, $S_p$ is an integer multiple of $C$ for every split. Hence if $N_n$ is a multiple of $C$, so are both $S_p$ and $N_n-S_p$, and the multiple-of-$C$ property propagates down the recursion. Starting from $C\mid N$, induction over the recursive partition tree gives that every terminal subset has size exactly $C$, so the leaf count is $N/C$ and $\eta=1$. Consequently, the patched representation $\bar{X}\in\mathbb{R}^{C\times P\times d}$ contains only real sensor nodes. When $C\nmid N$, writing $N=qC+m$ with $0<m<C$, the same argument confines the residual $m$ nodes to a single terminal subset, which is the only under-filled leaf.}}

\begin{definition}[\textbf{Node Balance Degree}]
\label{def:Node Balance Degree}
{
For a region containing $N$ traffic nodes, after a partitioning operation, let $L$ and $R$ ($L + R = N$) denote the number of nodes in the left and right subregions, respectively. The node balance degree $\delta$ is defined as follows.
\begin{equation}
\delta = |L - R|,
\end{equation}
where a smaller $\delta$ indicates a more balanced split. A large $\delta$ is undesirable because it creates skewed recursive trees, uneven patch workloads, and possible padding overhead.
}
\end{definition}

\begin{theorem}[\textbf{Split Balance}]
\label{theorem:Node Balance Degree}
{
In Square Partition, {for any region with $N_n>C$,} the node balance degree of each split satisfies
\begin{equation}
{\delta \le C,}
\end{equation}
where $C$ is the predefined patch capacity{, and the bound is tight}.
}
\end{theorem}

{\textit{Proof sketch.} By Eq.~\ref{splitpoint}, $S_p$ is the capacity multiple nearest to $N_n/2$, hence $|S_p - N_n/2| \le C/2$ and $\delta = |2S_p - N_n| \le C$. Tightness: for $N_n = 3C$ the nearest multiple is $C$ (or $2C$), giving $\delta = C$.}

\begin{definition}[\textbf{Region Aspect Ratio}]
\label{def:Region Aspect Ratio}
{
For a generated spatial region, the aspect ratio $\alpha$ is defined as
\begin{equation}
\alpha = \frac{\max(\Delta_{\mathit{Lng}}, \Delta_{\mathit{Lat}})}{\min(\Delta_{\mathit{Lng}}, \Delta_{\mathit{Lat}})},
\end{equation}
where $\Delta_{\mathit{Lng}}$ and $\Delta_{\mathit{Lat}}$ are the longitude and latitude spans. This metric serves as an empirical compactness diagnostic.
}
\end{definition}
% BLUE-EDIT: replaced the invalid aspect-ratio monotonicity theorem with a
% strictly correct axis-selection proposition; compactness remains empirical.
\begin{proposition}[\textbf{{Longest-Span Axis Selection}}]
\label{prop:axis-selection}
{
Consider a current region with coordinate spans $\Delta_{\max}\ge\Delta_{\min}$.
The rule in Eq.~\ref{splitaxis} selects an axis whose span equals $\Delta_{\max}$.
Therefore, whenever $\Delta_{\max}>\Delta_{\min}$, Square Partition does not
deliberately split along the strictly shorter axis. By contrast, a depth-alternating
{K-D Tree} provides no such per-region guarantee. In particular, if its scheduled axis
is the shorter one and a child retains the longer-axis span while its shorter-axis
span contracts by a factor $\rho\in(0,1)$, then that child's aspect ratio increases
from $\alpha$ to $\alpha/\rho$.
}
\end{proposition}

{
\textit{Proof.} The first statement follows directly from the maximization in
Eq.~\ref{splitaxis}. For the conditional {K-D Tree} statement, the specified child has
spans $\Delta_{\max}$ and $\rho\Delta_{\min}$, and hence aspect ratio
$\Delta_{\max}/(\rho\Delta_{\min})=\alpha/\rho>\alpha$.

Proposition~\ref{prop:axis-selection} identifies a failure mode avoided by the
axis-selection rule. A child's observed bounding box also
depends on the spatial support of the selected nodes, so compactness is
evaluated empirically in Table~\ref{exp:partition_quality}.
}

\subsection{Analysis of Low-rank Projection}
\label{subsubsection:low_rank}
We analyze the low-rank projection in the {Inter-Patch Communication} from two angles: its approximation error on low-rank dependencies, and the expressiveness of its function class relative to linear baselines.
\begin{theorem}[\textbf{Low-Rank Approximation Capacity}]
\label{theorem:low_rank}
{
Let $M \in \mathbb{R}^{P \times P}$ be an ideal global dependency matrix with singular values $\sigma_1 \ge \sigma_2 \ge \dots \ge \sigma_P \ge 0$. Among all rank-$r$ matrices expressible as $\hat{M} = \Theta_{r1}\Theta_{r2}$, where $\Theta_{r1} \in \mathbb{R}^{P \times r}$ and $\Theta_{r2} \in \mathbb{R}^{r \times P}$, the minimum Frobenius approximation error is
\begin{equation}
\min_{\Theta_{r1}, \Theta_{r2}} \| M - \Theta_{r1} \Theta_{r2} \|_F = \sqrt{\sum_{k=r+1}^{P} \sigma_k^2}
\end{equation}
This result means that the factorized projection can represent the best rank-$r$ approximation of $M$; when the empirical singular spectrum decays rapidly, the discarded tail energy is small.
}
\end{theorem}

\textit{Proof sketch.} {Since $\hat{M}$ is the product of two matrices with inner dimensionality $r$, $\operatorname{rank}(\hat{M})\le r$. By the Eckart-Young-Mirsky theorem, the best rank-$r$ approximation of $M$ is its truncated SVD $M_r=\sum_{i=1}^{r}\sigma_i u_i v_i^\top$. The residual satisfies
$$
\| M - M_r \|_F^2 = \left\| \sum_{k=r+1}^{P} \sigma_k u_k v_k^\top \right\|_F^2 = \sum_{k=r+1}^{P} \sigma_k^2
$$
because the singular vectors are orthonormal.}

% BLUE-EDIT: replaced the unsupported baseline-subsumption theorem with a
% structural result that exactly characterizes the shared inter-patch map.
\begin{proposition}[\textbf{{Shared Factorized Inter-Patch Operator}}]
\label{prop:structured-mixing}
{
Fix one HLI layer, sample, temporal patch, and feature channel, and consider the
bias-free factorized inter-patch projection before residual addition. Assume
$N=CP$, $1\le r\le P$, that no nonlinearity is inserted between the two factors, and that the same matrix
$M=\Theta_{r1}\Theta_{r2}\in\mathbb{R}^{P\times P}$ is applied independently to
each of the $C$ within-patch indices. If the node vector is ordered as
$x=[x_1^\top,\ldots,x_C^\top]^\top$, where $x_c\in\mathbb{R}^{P}$ collects the
$c$-th index across all patches, then the induced node-level linear map is
\begin{equation}
y=(I_C\otimes M)x.
\end{equation}
This projection in one layer uses $2Pr$ trainable matrix
parameters, has rank $C\,\operatorname{rank}(M)\le Cr$, and can be applied in
$O(CPr)=O(Nr)$ operations per feature channel without materializing an
$N\times N$ matrix, where $N=CP$.
}
\end{proposition}

{
\textit{Proof.} Applying $M$ independently gives
\begin{equation}
y=[(Mx_1)^\top,\ldots,(Mx_C)^\top]^\top=(I_C\otimes M)x.
\end{equation}
The two factors contain $Pr+rP=2Pr$ matrix parameters. The rank identity follows
from $\operatorname{rank}(A\otimes B)=\operatorname{rank}(A)\operatorname{rank}(B)$,
and factorized multiplication costs $O(Pr)$ for each of the $C$ index vectors.

Under patch-major rather than index-major vectorization, the same operator is
written as $M\otimes I_C$; the two forms are permutation-similar and have the
same rank and cost. The proposition concerns only the isolated factorized
projection. Residual addition and nonlinear feature transformations are not
rank-bounded by this statement. For $d$ feature channels, its per-layer cost is
$O(dCPr)=O(dNr)$, with the corresponding batch and temporal factors multiplying
this cost.

Theorem~\ref{theorem:low_rank} bounds the approximation error of the best
rank-$r$ factorization for a \emph{fixed} target matrix, while
Proposition~\ref{prop:structured-mixing} gives the exact form and cost of that
operator once it is shared across patches. The rank required in practice is
selected per dataset in \appref{app:Hyper-parameter}.
}
% BLUE-EDIT: removed a stray Markdown code fence from the LaTeX source.

\subsection{Complexity Analysis}
\label{proof:Complexity}
We denote the number of nodes, the number of nodes within each patch, the number of patches, the rank in the low-rank projection, and the hidden dimensionality by $N$, $C$, $P$, $r$, and $d$, respectively. Square Partition has $O(N \cdot ({\log N})^2)$ time complexity. Thus, constructing the balanced binary tree can be done efficiently in a pre-processing step. Next, we analyze the time and space complexity of SqLinear, focusing on the Hierarchical Linear Interaction (HLI) block. 

\textbf{Time Complexity.}
The time complexity of the HLI block is primarily determined by the {FeatureMixer} operations. The \textit{{Intra-Patch Refinement}} module applies a {FeatureMixer} within each of the $P$ patches, each of size $C$, resulting in a complexity of $O(P \cdot C \cdot d^2)$. The \textit{{Inter-Patch Communication}} module also involves a {FeatureMixer}, also with complexity $O(P \cdot C \cdot d^2)$.
In addition, the {Inter-Patch Communication} module includes a low-rank projection mechanism for global communication, with {complexity $O(C \cdot P \cdot d \cdot r) = O(N \cdot d \cdot r)$, which is dominated by $O(N \cdot d^2)$ since $r \le d$ in all our settings. Using $N = P \cdot C$, the time complexity becomes $O(N \cdot d^2)$.}
Thus, SqLinear scales linearly with respect to the number of nodes.

\textbf{Space Complexity.}
The space complexity is governed by the largest tensors stored during the forward pass, which are of dimensionality $N \times d$. Hence, the space complexity is $O(N \cdot d)$.

% {Overall, the analysis clarifies the role of Square Partition without overstating geometric optimality. Grid and Quadtree generate geometrically regular cells, but may waste computation on empty or under-filled regions; K-D Tree tightly balances node counts, but may create elongated cells due to alternating-axis splitting. Square Partition keeps the defensible properties that matter for SqLinear's tensorized forecasting pipeline: padding-free utilization under the selected capacities, bounded split imbalance, and empirically compact patches encouraged by longest-span axis selection.}

%% file: Appendix/3_Exp.tex
\section{Additional Experiments}
\label{app:exp}

\begin{table}[!tb]
    \centering
    \caption{Dataset statistics.}
    \vspace{-3mm}
    \resizebox{1.0\linewidth}{!}{
    \begin{tabular}{lcccc}
    \toprule
    Datasets & \#Nodes & \#Edges & \#TimeSlices & Timespan\\
    \midrule
    SD & 716 & 17,319 & 525,888 & 01/01/2017--12/31/2021\\
    GBA & 2,352 & 61,246  & 525,888  & 01/01/2017--12/31/2021\\
    GLA & 3,834 & 98,703 & 525,888 & 01/01/2017--12/31/2021\\
    CA & 8,600 & 201,363 & 525,888 & 01/01/2017--12/31/2021\\
    \bottomrule
    \end{tabular}}
    \label{exp:data}
\end{table}

\subsection{Detailed Experimental Setup}
\label{app:experimental-setup}
\subsubsection{Datasets Description}
\label{exp:datasets}
We conduct experiments on four large-scale benchmark datasets, SD, GBA, GLA, and CA, whose graph sizes range from 716 to 8,600 sensor nodes. Following established traffic forecasting protocols, each dataset is split chronologically into training, validation, and testing sets with a 6:2:2 ratio. The standard forecasting setting uses 12 historical steps to predict the next 12 steps. To examine temporal scalability, we additionally use 96 historical steps and forecast 48, 96, 192, and 672 future steps on SD, GBA, GLA, and CA. Dataset statistics are summarized in Table~\ref{exp:data}.

\subsubsection{Implementation Details}
\label{exp:Implementations}
SqLinear is trained using the AdamW optimizer~\cite{loshchilov2017decoupled} with an initial learning rate of 0.001 and a weight decay of 0.0001.
%The learning rate is adjusted by ReduceLROnPlateau~\cite{schmidhuber2015deep} with a patience of five epochs.
%The input projection dimension is set to 64, while the day-of-week, timeslice-of-day, and spatial embedding dimensions are set to 32. \textcolor{blue}{The leaf node capacity is set to 4, 48, 54, and 86 for SD, GBA, GLA, and CA, respectively. The number of Hierarchical Linear Interaction layers is set to 1, 2, 4, and 4 for these datasets, respectively.} Hyperparameter effects are analyzed in Section~\ref{app:Hyper-parameter}.
We follow the official baseline configurations~\cite{liu2024largest} and only adjust the input and prediction windows to match the corresponding forecasting protocol. Long-Horizon prediction experiments are implemented on the BasicTS~\cite{Basic} framework. All experiments are implemented in PyTorch and conducted on an NVIDIA RTX A40 GPU with 48 GB of memory. For efficiency analysis, runtime denotes the measured training wall-clock time of a run, and inference latency denotes the measured evaluation latency.

\subsubsection{Baselines}
For the standard setting, we compare SqLinear with {11 SOTA baselines} from three families. \textbf{(i) The GNN-based group} includes GWNET~\cite{wu2019graph}, AGCRN~\cite{bai2020adaptive}, D2STGNN~\cite{shao2022decoupled}, and DSTAGNN~\cite{lan2022dstagnn}. \textbf{(ii) The linear-model group} includes STID~\cite{shao2022spatial}, BigST~\cite{han2024bigst}, and RPMixer~\cite{yeh2024rpmixer}. \textbf{(iii) The attention- and expert-based group} includes STAEformer~\cite{liu2023spatio}, STWave~\cite{fang2023spatio}, FaST~\cite{zhao2026fast}, and PatchSTG~\cite{fang2024efficient}{.} For the long-horizon study, we choose strong or lightweight scalability-oriented representatives, including STID, RPMixer, BigST, and PatchSTG, so that the comparison directly tests temporal scalability without overcrowding the main text.

\subsection{Long-Horizon Prediction Performance}
\label{app:scalability-metrics}
Table~\ref{tab:scalability_metrics} reports the detailed MAE, RMSE, and MAPE values for the long-horizon scalability study. SqLinear achieves the best performance on all three metrics across all 16 dataset-horizon combinations. The gains remain stable as the prediction length increases, and SqLinear retains the lowest error on every metric even under the longest $96{\Rightarrow}672$ setting. These results confirm that SqLinear preserves accurate spatio-temporal forecasting when the temporal horizon is substantially enlarged.

\begin{table*}[tb]
\centering
\caption{Detailed long-horizon scalability results. \redval{Bold} indicates the best performance and \blueval{underline} denotes the second best performance. The notation $96{\Rightarrow}48$ means using the past 96 time steps to predict the next 48 steps; MAPE is reported in percentage.}
\vspace{-3mm}
\label{tab:scalability_metrics}
\setlength{\tabcolsep}{3pt}
\renewcommand{\arraystretch}{0.95}
\resizebox{0.85\textwidth}{!}{
\begin{tabular}{c|l|ccc|ccc|ccc|ccc}
\toprule
\multirow{2}{*}{Data} & \multirow{2}{*}{Method} &
\multicolumn{3}{c|}{$96{\Rightarrow}48$} &
\multicolumn{3}{c|}{$96{\Rightarrow}96$} &
\multicolumn{3}{c|}{$96{\Rightarrow}192$} &
\multicolumn{3}{c}{$96{\Rightarrow}672$} \\
\cmidrule(lr){3-5}\cmidrule(lr){6-8}\cmidrule(lr){9-11}\cmidrule(lr){12-14}
 & & MAE & RMSE & MAPE &
     MAE & RMSE & MAPE &
     MAE & RMSE & MAPE &
     MAE & RMSE & MAPE \\
\midrule
\multirow{6}{*}{SD}
& STID & 20.89 & 39.14 & 14.67 & 22.71 & 42.81 & 15.90 & 24.77 & 46.91 & 17.59 & 27.14 & 50.71 & 19.28 \\
& RPMixer & 21.78 & 37.71 & 14.99 & 23.80 & 41.68 & 16.77 & 25.91 & 47.18 & 18.35 & 28.08 & 52.68 & 20.67 \\
& BigST & 21.98 & 37.43 & 16.41 & 24.64 & 42.92 & 18.66 & 26.90 & 47.77 & 20.46 & 29.63 & 52.29 & 22.47 \\
& PatchSTG & 20.60 & 39.45 & 14.23 & 23.21 & 45.10 & 16.17 & 26.06 & 49.99 & 17.91 & 28.16 & 53.78 & 19.48 \\
& FaST & \blueval{19.37} & \blueval{34.54} & \blueval{12.85} & \blueval{21.46} & \blueval{39.18} & \blueval{14.37} & \blueval{24.23} & \blueval{45.09} & \blueval{16.48} & \blueval{26.61} & \blueval{49.42} & \blueval{17.99} \\
& \textbf{SqLinear} & \redval{19.31} & \redval{33.99} & \redval{12.59} & \redval{20.89} & \redval{36.78} & \redval{13.74} & \redval{23.46} & \redval{42.39} & \redval{15.85} & \redval{26.19} & \redval{48.84} & \redval{17.91} \\
\midrule
\multirow{6}{*}{GBA}
& STID & 22.03 & 38.96 & 18.94 & 23.58 & 41.46 & 20.57 & \blueval{25.50} & \blueval{44.45} & 22.66 & 27.73 & 47.59 & 24.66 \\
& RPMixer & 23.47 & 40.18 & 19.36 & 24.03 & 41.27 & 20.61 & 25.92 & 44.78 & 22.10 & 27.94 & 47.69 & 25.02 \\
& BigST & 24.03 & 39.61 & 21.52 & 25.39 & 41.53 & 23.49 & 27.48 & 44.70 & 25.62 & 29.50 & 48.11 & 27.56 \\
& PatchSTG & \blueval{21.78} & \blueval{37.56} & 17.62 & 24.36 & 43.45 & 20.83 & 26.16 & 46.37 & 22.53 & 28.52 & 49.99 & 24.63 \\
& FaST & 22.00 & 38.71 & \blueval{17.10} & \blueval{23.39} & \blueval{41.13} & \blueval{19.85} & 25.97 & 44.95 & \blueval{21.68} & \blueval{27.12} & \blueval{47.36} & \blueval{22.96} \\
& \textbf{SqLinear} & \redval{20.99} & \redval{37.13} & \redval{17.02} & \redval{22.20} & \redval{39.31} & \redval{18.39} & \redval{23.77} & \redval{41.97} & \redval{20.25} & \redval{26.09} & \redval{46.06} & \redval{22.77} \\
\midrule
\multirow{6}{*}{GLA}
& STID & 22.34 & 40.52 & 15.25 & 23.98 & 43.64 & 16.68 & 26.21 & 47.02 & 18.86 & 28.91 & 51.16 & 20.83 \\
& RPMixer & 22.91 & \blueval{38.13} & 15.84 & 23.65 & 40.86 & 16.16 & 25.82 & \blueval{44.09} & 18.04 & 28.56 & 48.76 & 19.91 \\
& BigST & 24.43 & 41.01 & 18.09 & 25.79 & 43.27 & 19.85 & 28.15 & 46.76 & 21.72 & 30.76 & 50.95 & 24.17 \\
& PatchSTG & 21.75 & 39.55 & 14.59 & 23.80 & 43.33 & 16.41 & 26.49 & 48.62 & 18.32 & 29.51 & 53.01 & 20.86 \\
& FaST & \blueval{21.70} & 38.61 & \blueval{14.14} & \blueval{22.95} & \blueval{40.84} & \blueval{15.25} & \blueval{25.62} & 44.53 & \blueval{16.99} & \blueval{27.37} & \blueval{48.32} & \blueval{18.84} \\
& \textbf{SqLinear} & \redval{20.83} & \redval{36.80} & \redval{13.42} & \redval{22.42} & \redval{40.27} & \redval{14.97} & \redval{24.29} & \redval{43.43} & \redval{16.52} & \redval{26.89} & \redval{47.63} & \redval{18.57} \\
\midrule
\multirow{6}{*}{CA}
& STID & 20.61 & 37.67 & 16.38 & 22.09 & 40.09 & 17.96 & 24.29 & 43.88 & 20.09 & 26.54 & 47.29 & 22.00 \\
& RPMixer & 21.31 & 35.97 & 17.39 & 21.96 & \blueval{37.63} & 17.59 & 23.86 & 41.16 & 19.14 & 26.11 & 45.44 & 21.32 \\
& BigST & 22.32 & 37.74 & 18.60 & 24.24 & 40.69 & 21.08 & 26.38 & 43.63 & 23.22 & 28.38 & 47.32 & 24.72 \\
& PatchSTG & 20.67 & 37.99 & 16.38 & 22.64 & 41.76 & 18.04 & 24.51 & 44.73 & 20.13 & 28.42 & 50.94 & 23.20 \\
& FaST & \blueval{19.87} & \blueval{35.43} & \blueval{14.96} & \blueval{20.95} & 37.94 & \blueval{16.16} & \blueval{22.78} & \blueval{41.07} & \blueval{17.98} & \blueval{25.29} & \blueval{45.22} & \blueval{20.29} \\
& \textbf{SqLinear} & \redval{19.30} & \redval{35.27} & \redval{14.82} & \redval{20.53} & \redval{37.40} & \redval{15.95} & \redval{21.98} & \redval{39.19} & \redval{17.27} & \redval{24.79} & \redval{44.55} & \redval{19.97} \\
\bottomrule
\end{tabular}}
\end{table*}
\subsection{Hyperparameter Study (RQ5)}
\label{app:Hyper-parameter}
We study key hyperparameters on GBA and CA, as shown in Fig.~\ref{exp:Hyper}. Similar results were observed on other datasets.
\begin{figure}[!tb]
    \centering
    \begin{subfigure}{0.30\linewidth}
        \includegraphics[width=\linewidth]{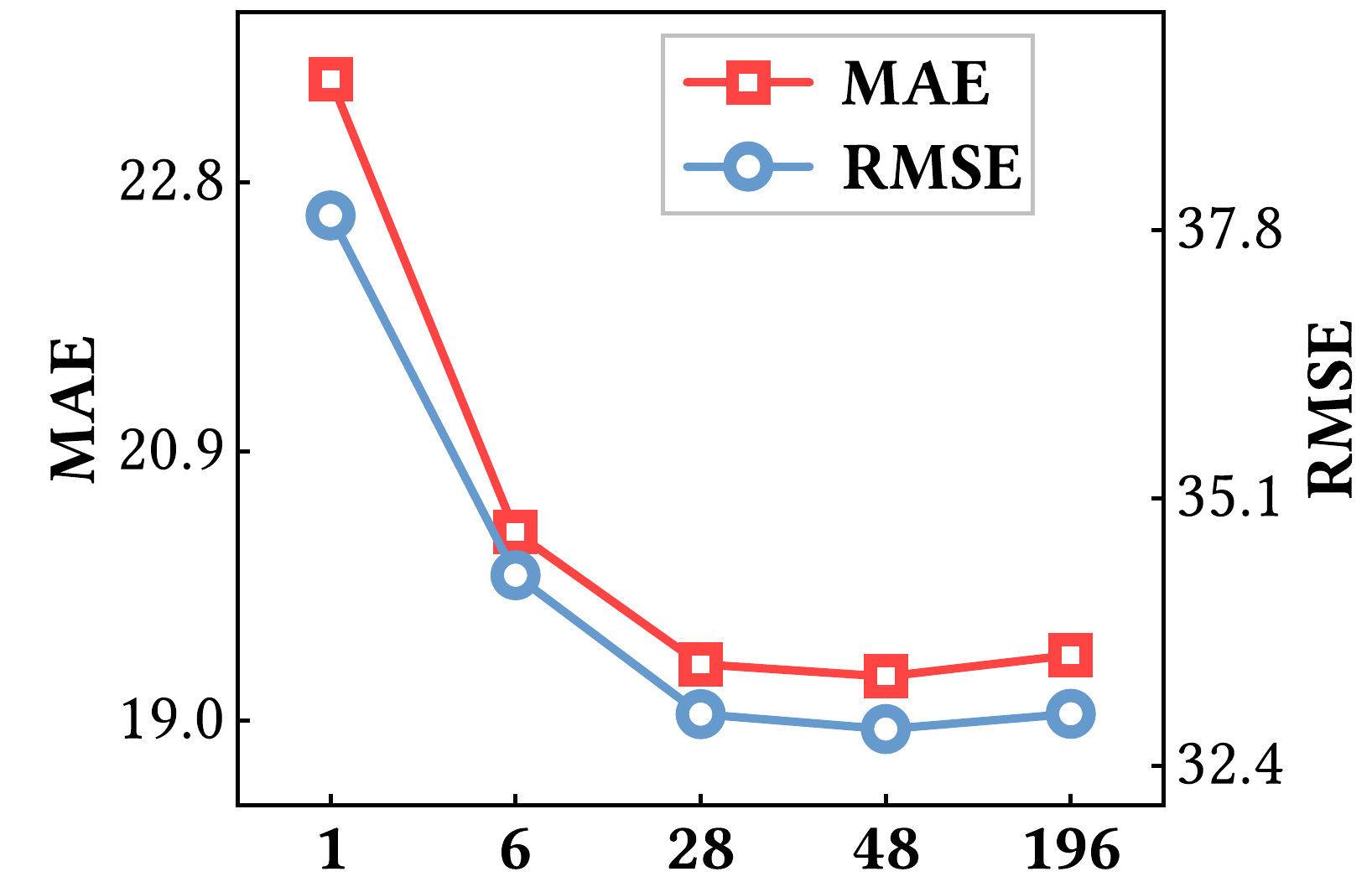}
        \caption{$C$ on GBA}
    \end{subfigure}
    \hfill
    \begin{subfigure}{0.30\linewidth}
        \includegraphics[width=\linewidth]{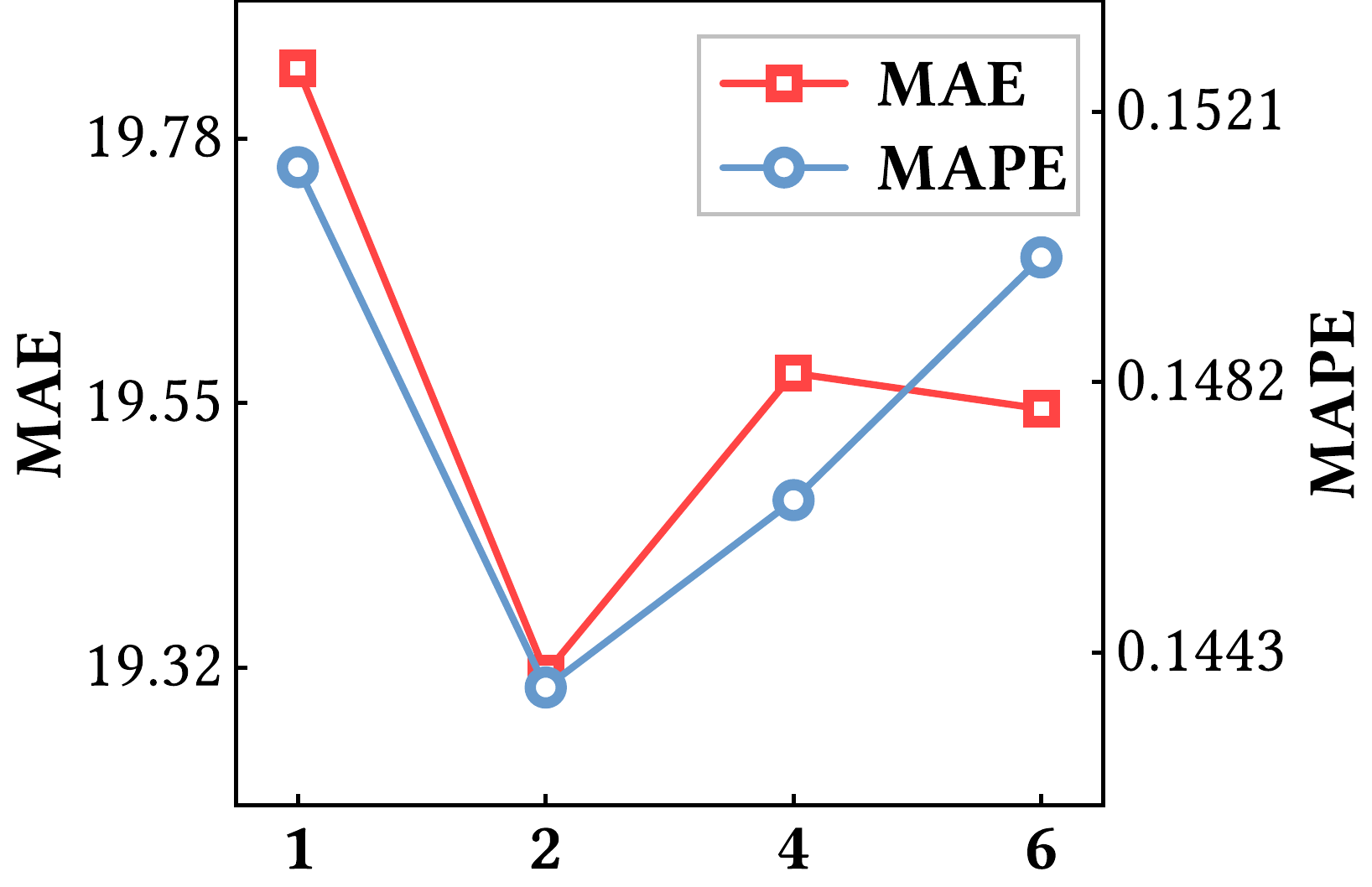}
        \caption{$L$ on GBA}
    \end{subfigure}
    \hfill
    \begin{subfigure}{0.30\linewidth}
        \includegraphics[width=\linewidth]{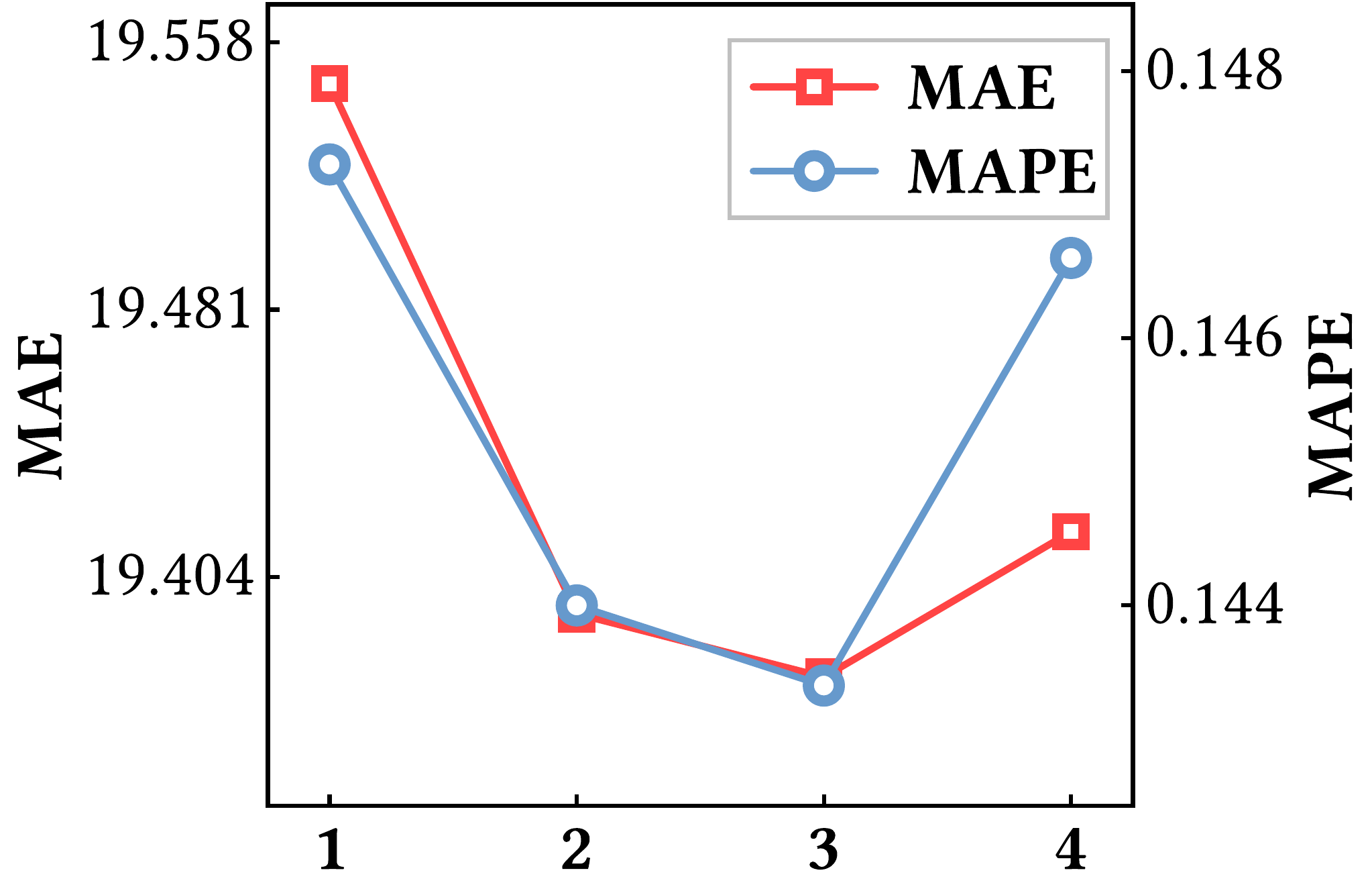}
        \caption{{$r$} on GBA}
    \end{subfigure}
    \begin{subfigure}{0.30\linewidth}
        \includegraphics[width=\linewidth]{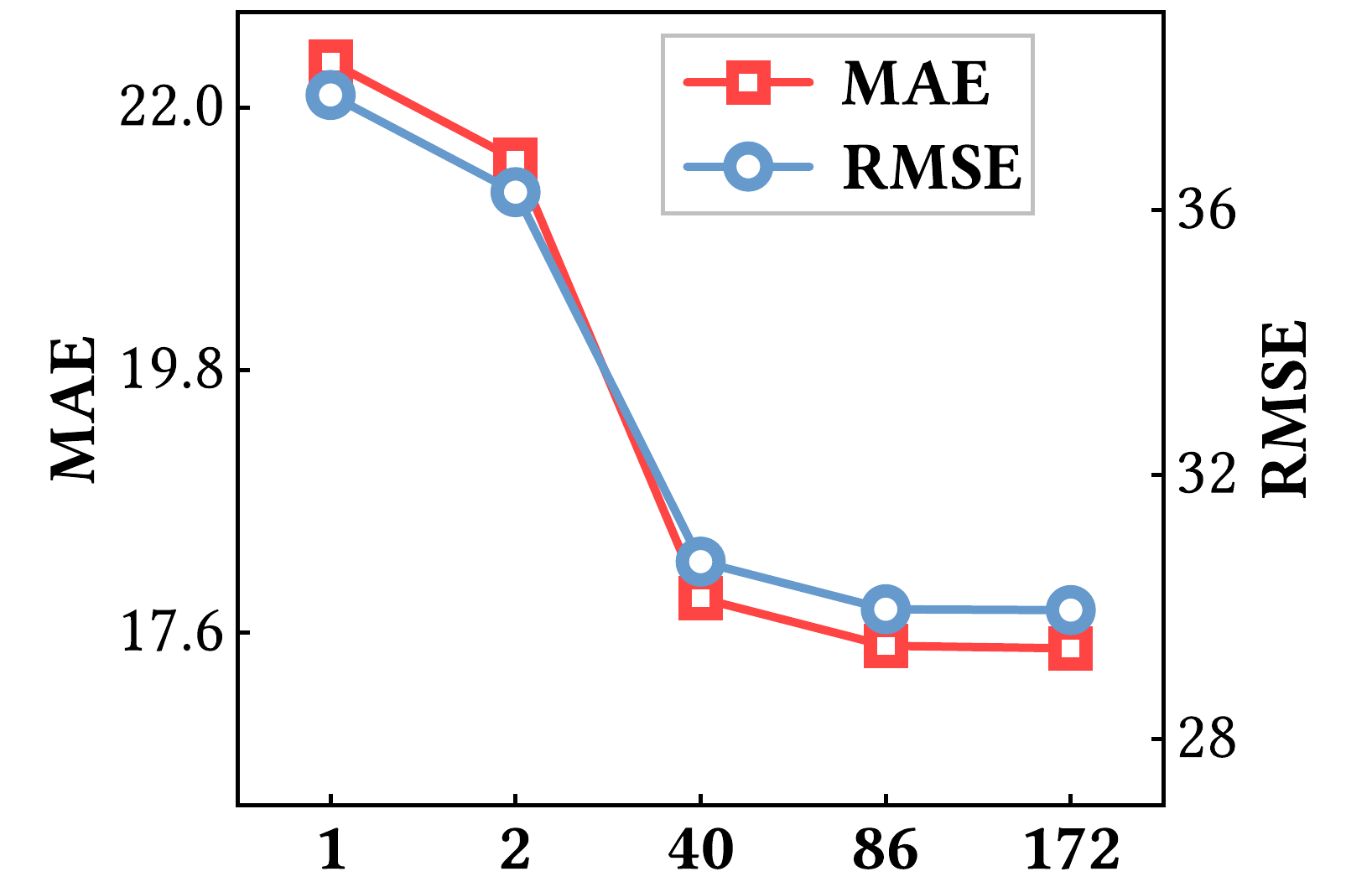}
        \caption{$C$ on CA}
    \end{subfigure}
    \hfill
    \begin{subfigure}{0.30\linewidth}
        \includegraphics[width=\linewidth]{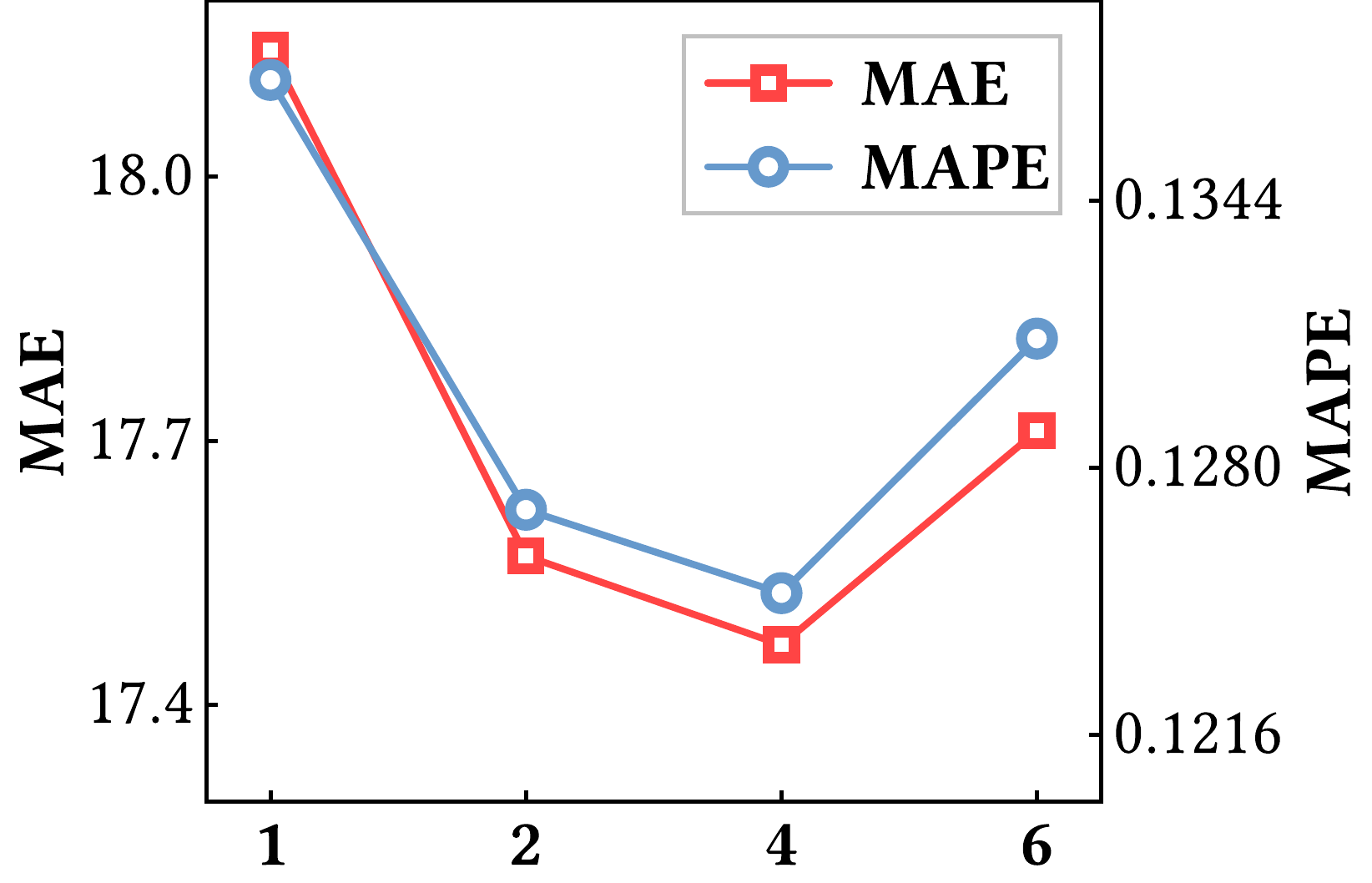}
        \caption{$L$ on CA}
    \end{subfigure}
    \hfill
    \begin{subfigure}{0.30\linewidth}
        \includegraphics[width=\linewidth]{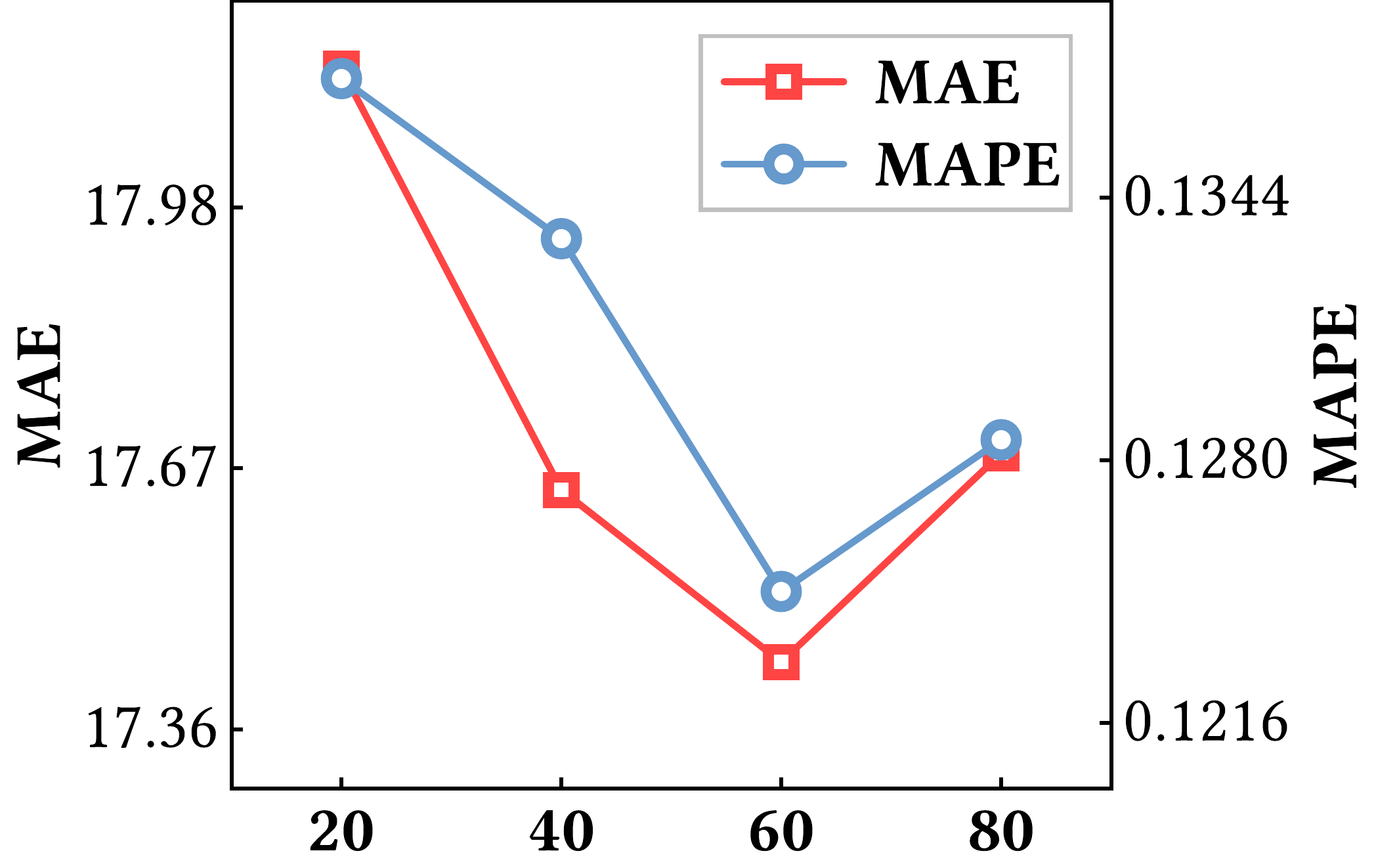}
        \caption{{$r$} on CA}
    \end{subfigure}
    \caption{Hyperparameter sensitivity analysis.}
    \label{exp:Hyper}
\end{figure}

\textbf{The number of nodes per patch $C$.} Performance peaks when the number of nodes per patch $C$ and the number of patches $P$ are numerically balanced. SqLinear obtains the best accuracy with $C=48$ on GBA and $C=86$ on CA. This pattern reflects {a trade-off in patch granularity: larger patches let each patch span a richer local neighborhood}, while too-large patches reduce the granularity of global interaction.

\textbf{The number of layers $L$.} The optimal depth depends on graph scale. Larger networks such as CA require deeper interaction layers to propagate information across broader spatial regions, whereas shallower models are sufficient for smaller graphs. Excessive depth can add unnecessary computation and may lead to overfitting on smaller datasets.

\textbf{The low-rank dimension {$r$}.} A moderate low-rank dimension is enough to encode latent spatial relations. Increasing {$r$} initially improves representation capacity, but overly large values introduce redundancy and increase memory usage without clear accuracy gains.

\subsection{Partition Strategy Analysis (RQ6)}
\label{sec:partition_strategy}
\textbf{Topology-Aware Extension.} We compare Square Partition with Geometry-First Road Partitioning (GFRP) and Topology-First Road Grouping (TFRG), whose definitions are provided in Section~\ref{ScalabilityAnalysis}. Fig.~\ref{Scalability} shows that Square Partition consistently achieves better forecasting accuracy. Although GFRP and TFRG explicitly preserve road topology, they can generate irregular or unbalanced groups that are less suitable for efficient tensorized modeling.
This result indicates that SqLinear benefits more from compact and balanced spatial patches than from hard-coded topological grouping. Square Partition provides local spatial context in each patch, while the Hierarchical Linear Interaction module learns long-range cross-region relations from data. This division of labor makes the overall architecture more flexible for large-scale traffic forecasting.
\begin{figure}[!tb]
    \vspace{0.1cm}
    \centering
    \begin{subfigure}{0.46\linewidth}
        \includegraphics[width=\linewidth]{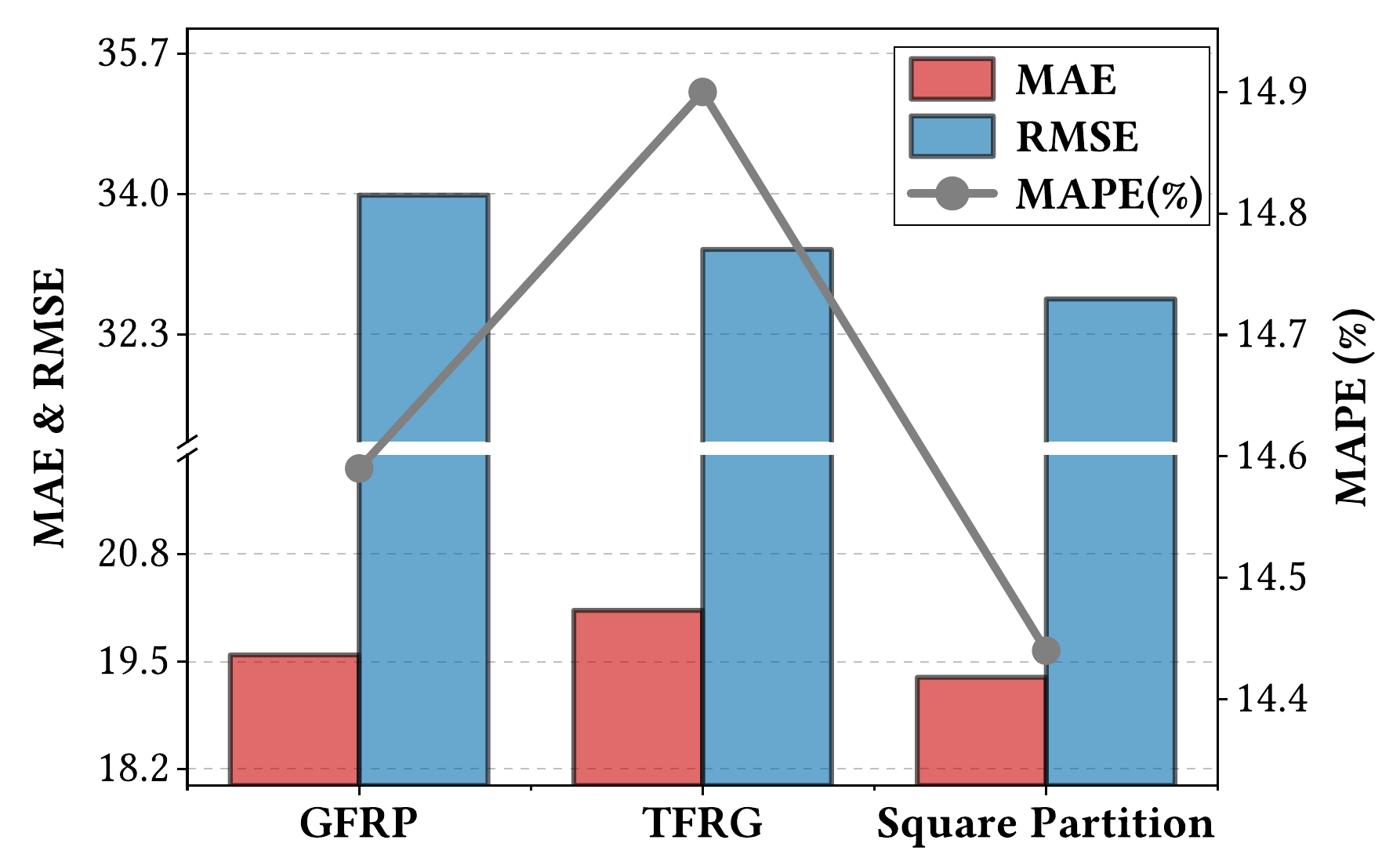}
        \caption{GBA dataset}
    \end{subfigure}
    \hfill
    \begin{subfigure}{0.46\linewidth}
        \includegraphics[width=\linewidth]{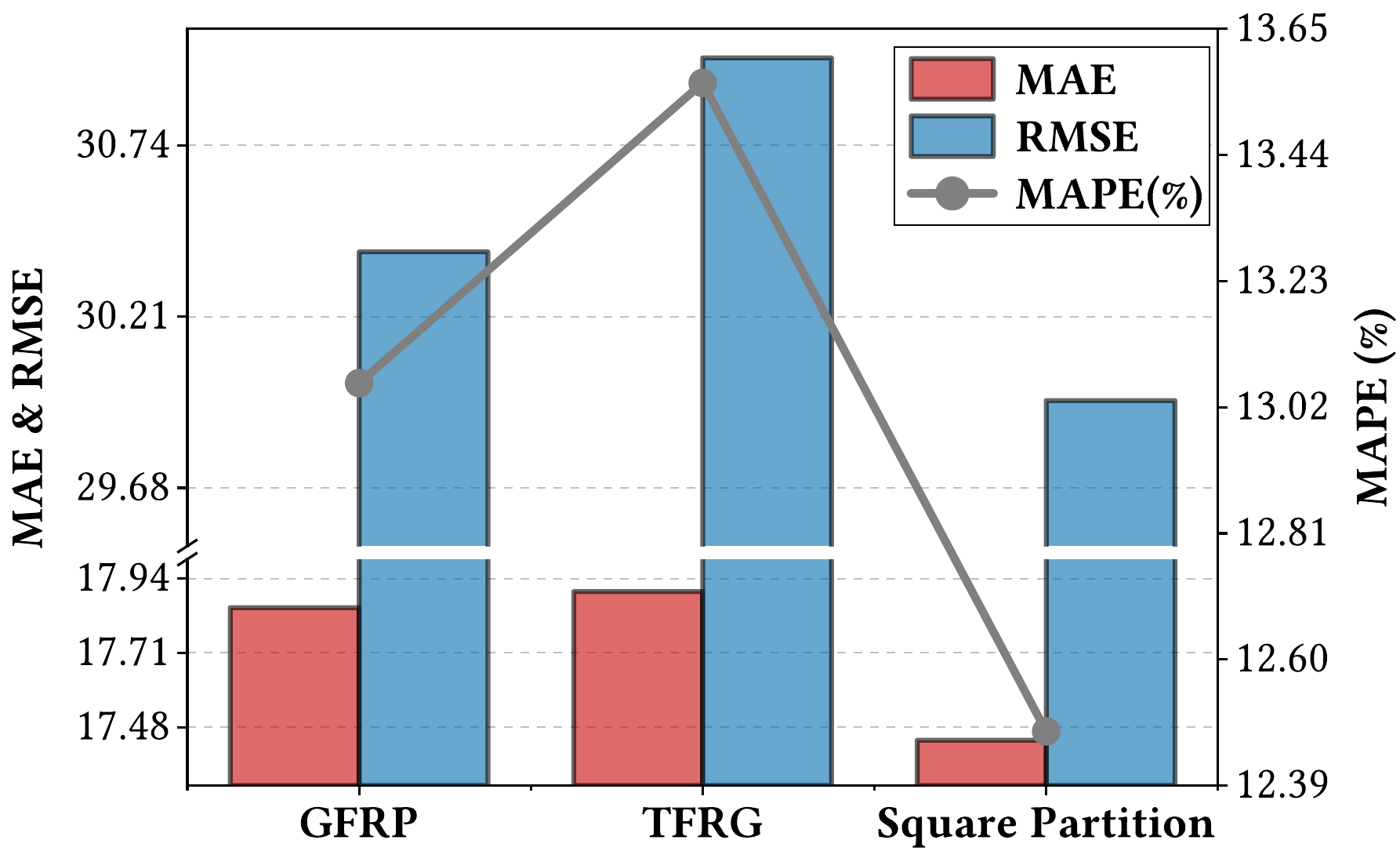}
        \caption{CA dataset}
    \end{subfigure}
    \caption{Partition strategy analysis.}
    \label{Scalability}
\end{figure}

\textbf{Empirical partition quality.} Table~\ref{exp:partition_quality} empirically grounds the partition properties analyzed in Section~\ref{subsubsec:partition_analysis} on the representative CA dataset, reporting the three quantities defined there: patch utilization $\eta$ (\textbf{Definition~\ref{def:Patch Utilization}}), the patch-size coefficient of variation (CV) as {a leaf-level dispersion measure complementary to the per-split balance degree $\delta$} (\textbf{Definition~\ref{def:Node Balance Degree}}), and region aspect ratio $\alpha$ (\textbf{Definition~\ref{def:Region Aspect Ratio}}) for compactness.
Square Partition is the only method that attains both full utilization ($\eta=1$) and perfect balance ($\text{CV}=0$) on every dataset, {both of which follow from \textbf{Theorem~\ref{theorem:Patch Utilization}}, since equal-sized leaves imply $\text{CV}=0$ and $\eta=1$}. For compactness, its longest-span splitting rule keeps the p90 aspect ratio strictly below that of the K-D Tree on all four datasets (e.g., $13.47$ vs.\ $21.89$ on SD and $5.68$ vs.\ $9.56$ on GLA), because it never subdivides the already-shorter side, whereas the alternating-axis rule can amplify elongation in the tail. Grid and Quadtree occasionally achieve a lower aspect ratio on small datasets, but only by packing nodes into a few oversized cells, which harms both balance (CV up to $1.59$) and utilization (down to $0.16$). These measurements confirm that Square Partition realizes the utilization and balance guarantees of the analysis while remaining the most compact among capacity-balanced partitions in practice.

\begin{table}[tb]
\centering
\caption{Empirical partition quality on the four datasets. \textbf{Bold} marks the best per column within each dataset.}
\label{exp:partition_quality}
\setlength{\tabcolsep}{5pt}
\begin{tabular}{l l c c c c}
\toprule
\multirow{2}{*}{Dataset} & \multirow{2}{*}{Method} & Util.\ $\uparrow$ & Balance $\downarrow$ & \multicolumn{2}{c}{Aspect Ratio $\downarrow$} \\
\cmidrule(lr){5-6}
 & & $\eta$ & CV & median & p90 \\
\midrule
\multirow{4}{*}{SD}
 & Grid     & 0.260 & 0.866 & \textbf{1.83} & \textbf{5.44} \\
 & Quadtree & 0.630 & 0.404 & 2.76 & 14.87 \\
 & K-D Tree & 0.932 & 0.088 & 4.00 & 21.89 \\
 & \textbf{Square} & \textbf{1.000} & \textbf{0.000} & 2.65 & 13.47 \\
\midrule
\multirow{4}{*}{GBA}
 & Grid     & 0.327 & 0.836 & \textbf{1.27} & 2.29 \\
 & Quadtree & 0.445 & 0.561 & 1.56 & 6.65 \\
 & K-D Tree & 0.993 & 0.012 & 1.86 & 5.47 \\
 & \textbf{Square} & \textbf{1.000} & \textbf{0.000} & 1.85 & \textbf{4.62} \\
\midrule
\multirow{4}{*}{GLA}
 & Grid     & 0.161 & 1.594 & 2.24 & 14.63 \\
 & Quadtree & 0.488 & 0.503 & 2.22 & 10.73 \\
 & K-D Tree & 0.998 & 0.005 & 2.11 & 9.56 \\
 & \textbf{Square} & \textbf{1.000} & \textbf{0.000} & \textbf{1.68} & \textbf{5.68} \\
\midrule
\multirow{4}{*}{CA}
 & Grid     & 0.172 & 1.483 & 1.55 & \textbf{4.53} \\
 & Quadtree & 0.413 & 0.609 & \textbf{1.47} & 5.36 \\
 & K-D Tree & 0.988 & 0.006 & 2.06 & 4.66 \\
 & \textbf{Square} & \textbf{1.000} & \textbf{0.000} & 1.70 & 4.54 \\
\bottomrule
\end{tabular}
\end{table}

\subsection{Case Study (RQ7)}
\label{app:case-study}
% \begin{figure}[tb]
%   \centering
%   \includegraphics[width=\linewidth]{pics/SDshowcase.pdf}
%   \caption{Visualization of Partitioning results on the SD dataset.}
%   \label{Visualization:SDshowcase}
% \end{figure}
\begin{figure}[tb]
  \centering
  \includegraphics[width=\linewidth]{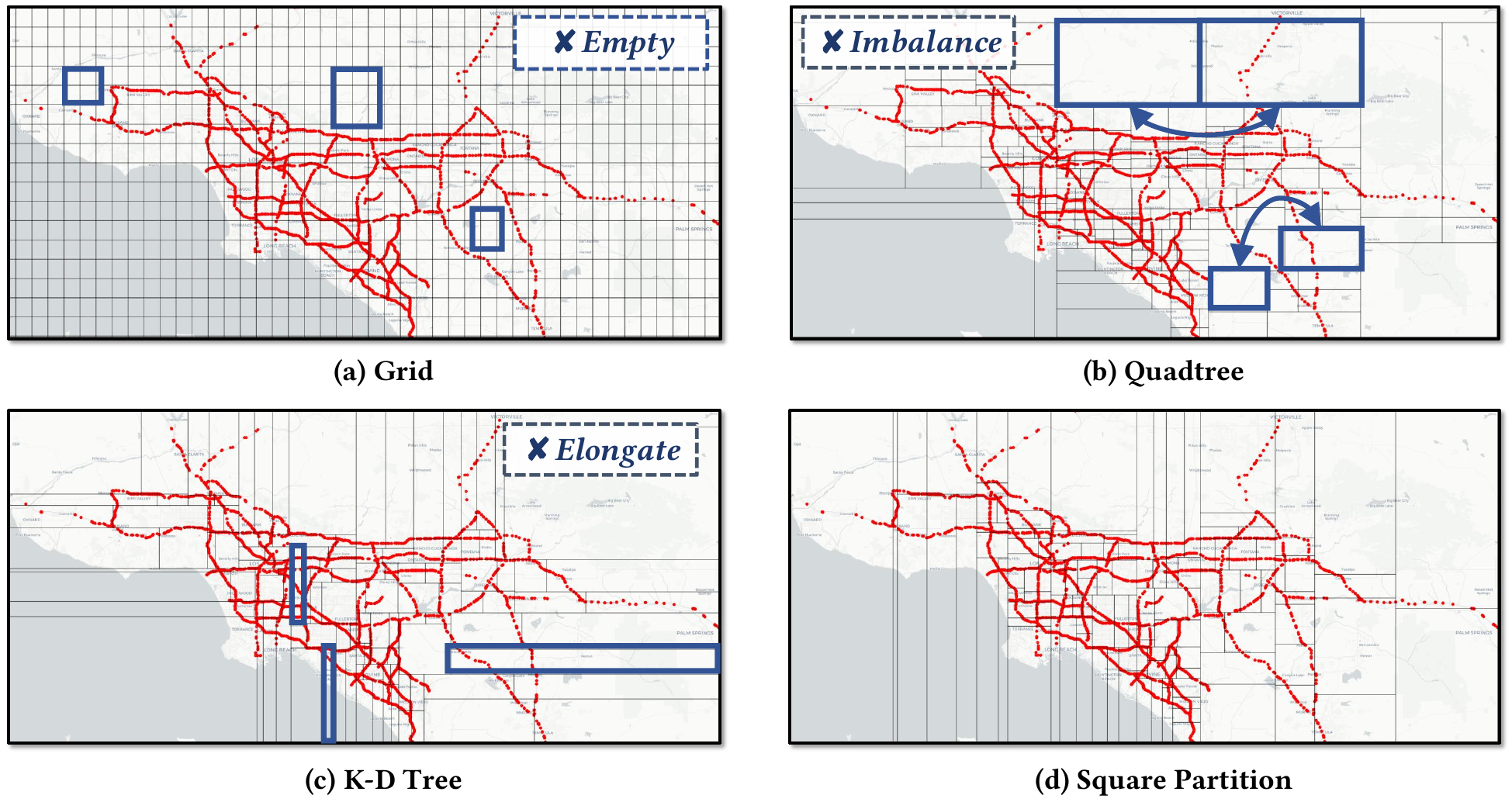}
  \caption{Visualization of partitioning results.}
  \label{Visualization:GLAshowcase}
\end{figure}
To provide an intuitive understanding of SqLinear, we present a case study visualizing the partitioning and the patterns learned by the method. While this study is based on the SD and GLA, the key observations and learned patterns also apply to the other datasets.

\textbf{Visual Analysis of Spatial Partitioning.} Fig.~\ref{Visualization:GLAshowcase} displays the partitioning results, while comparing also with the Grid, Quadtree, and K-D Tree approaches. The blue annotations highlight the limitations of the Grid, Quadtree, and K-D Tree approaches. The case study illustrates that Square Partition consistently generates contiguous, balanced, strictly non-overlapping, and approximately square spatial patches without requiring padding.

\textbf{Interpretability of Hierarchical Interaction.} Fig.~\ref{exp:casestudy}(a) visualizes the learned weights of the {inter-patch communication module}, while Fig.~\ref{exp:casestudy}(b) shows {the correlation among the learned node representations} within the 118th patch. Fig.~\ref{exp:casestudy}(c) shows the real-world locations of the 118th patch and the 473rd node, providing geographical context. {Given that the inter-patch weights are directly learned and the node representations are directly observable, both} offer transparent means of interpreting the internal operational mechanisms of SqLinear. A key observation is that the 118th patch, which corresponds to a major arterial intersection, is assigned a high importance weight. Furthermore, within this patch, {the central 473rd node's representation correlates more strongly with its downstream nodes than with others}. This pattern aligns closely with real-world traffic propagation dynamics, thereby validating the model's ability to capture meaningful spatio-temporal dependencies.
\begin{figure}[h]
    \centering
    \includegraphics[width=\linewidth]{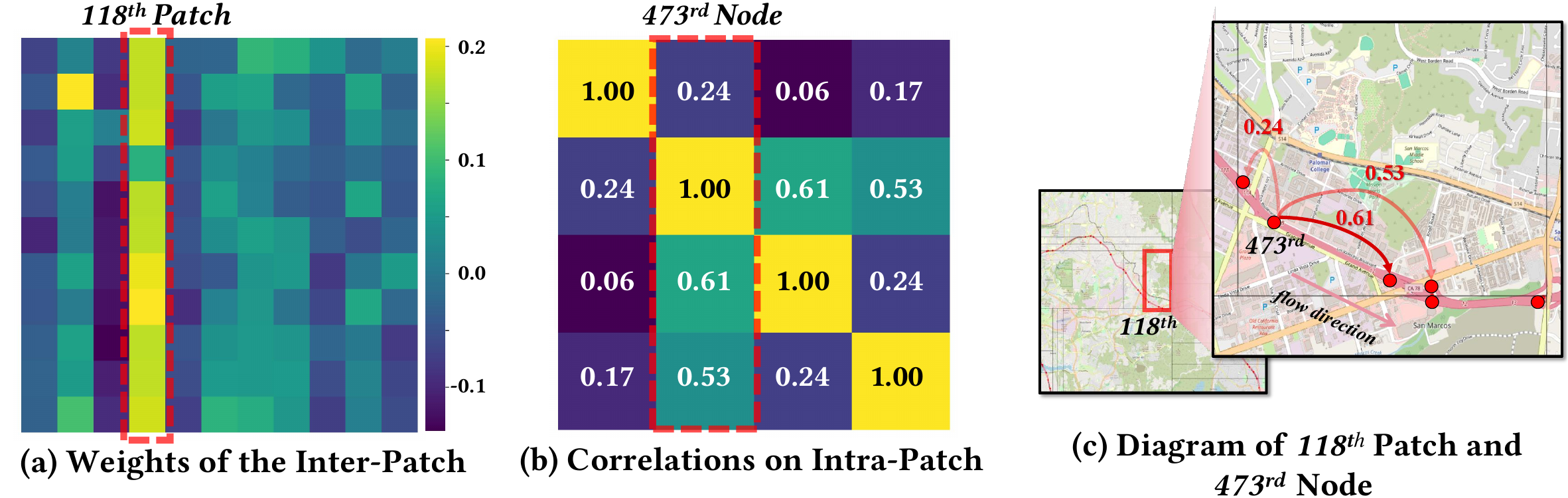}
      \caption{Visualization of learned spatial dependencies.}
      \label{exp:casestudy}
\end{figure}

\subsection{Limitations and Future Work}
\label{app:limitations}
While SqLinear achieves state-of-the-art performance across four large-scale datasets, limitations suggest promising directions for future research. The linear interaction module is optimized for efficiency, but it may have limited capacity to model abrupt and highly nonlinear dynamics, such as those caused by traffic accidents. Consequently, investigating hybrid linear--nonlinear architectures that balance computational efficiency and expressive power is a critical avenue for enhancing model robustness.